\documentclass[twoside,11pt]{article}

\usepackage{style/jmlr2e}

\usepackage{style/macros}

\jmlrheading{18}{2017}{1-67}{12/15; Revised 12/16}{10/17}{15-631}{Stephen H.\ Bach, Matthias Broecheler, Bert Huang, and Lise Getoor}

\ShortHeadings{Hinge-Loss Markov Random Fields and Probabilistic Soft Logic}{Bach, Broecheler, Huang, and Getoor}
\firstpageno{1}

\begin{document}

\title{Hinge-Loss Markov Random Fields \\and Probabilistic Soft Logic}

\author{\name Stephen H.\ Bach \email bach@cs.stanford.edu \\
       \addr Computer Science Department \\
       Stanford University \\
       Stanford, CA 94305, USA
       \AND
       \name Matthias Broecheler \email matthias@datastax.com \\
       \addr DataStax
       \AND
       \name Bert Huang \email bhuang@vt.edu \\
       \addr Computer Science Department \\
       Virginia Tech\\
       Blacksburg, VA 24061, USA
       \AND
       \name Lise Getoor \email getoor@soe.ucsc.edu \\
       \addr Computer Science Department \\
       University of California, Santa Cruz \\
       Santa Cruz, CA 95064, USA}

\editor{Luc De Raedt}

\maketitle


\begin{abstract}
A fundamental challenge in developing high-impact machine learning technologies is balancing the
need to model rich, structured domains with the ability to scale to big data.
Many important problem areas are both richly structured and large scale, from social and biological networks,
to knowledge graphs and the Web, to images, video, and natural language.
In this paper, we introduce two new formalisms for modeling structured data, and show that they can both capture rich structure and scale to big data.
The first, hinge-loss Markov random fields (HL-MRFs), is a new kind of probabilistic graphical model that generalizes different
approaches to convex inference.
We unite three approaches from the randomized algorithms, probabilistic graphical models,
and fuzzy logic communities, showing that all three lead to the same inference objective.
We then define HL-MRFs by generalizing this unified objective.
The second new formalism, probabilistic soft logic (PSL), is a probabilistic programming language
that makes HL-MRFs easy to define using a syntax based on first-order logic.
We introduce an algorithm for inferring most-probable variable assignments (MAP inference)
that is much more scalable than general-purpose convex optimization methods,
because it uses message passing to take advantage of sparse dependency structures.
We then show how to learn the parameters of HL-MRFs.
The learned HL-MRFs are as accurate as analogous discrete models, but much more scalable.
Together, these algorithms enable HL-MRFs and PSL to model rich, structured data at scales not previously possible.
\end{abstract}

\begin{keywords}
Probabilistic graphical models, statistical relational learning, structured prediction
\end{keywords}


\section{Introduction}

In many problems in machine learning, the domains are rich and structured, with many interdependent
elements that are best modeled jointly.
Examples include social networks, biological networks, the Web, natural language, computer vision,
sensor networks, and so on.
Machine learning subfields such as statistical relational learning \citep{getoor:book07},
inductive logic programming \citep{muggleton:lp94}, and structured prediction \citep{bakir:book07}
all seek to represent dependencies in data induced by relational structure.
With the ever-increasing size of available data, there is a growing need for models that are highly
scalable while still able to capture rich structure.

In this paper, we introduce \emph{hinge-loss Markov random fields} (HL-MRFs),
a new class of probabilistic graphical models designed to enable scalable modeling of rich, structured data.
HL-MRFs are analogous to discrete MRFs, which are undirected probabilistic graphical models in which probability
mass is log-proportional to a weighted sum of feature functions.
Unlike discrete MRFs, however, HL-MRFs are defined over continuous variables in the $[0,1]$ unit interval.
To model dependencies among these continuous variables, we use linear and quadratic hinge functions, so that probability density is lost according to a weighted sum of hinge losses.
As we will show, hinge-loss features capture many common modeling patterns for structured data.

When designing classes of models, there is generally a trade off between scalability and expressivity:
the more complex the types and connectivity structure of the dependencies, the more computationally
challenging inference and learning become.
HL-MRFs address a crucial gap between the two extremes.
By using hinge-loss functions to model the dependencies among the variables, which admit
highly scalable inference without restrictions on their connectivity structure, HL-MRFs can
capture a wide range of useful relationships.
One reason they are so expressive is that hinge-loss dependencies are at the core of a number of
scalable techniques for modeling both discrete and continuous structured data.

To motivate HL-MRFs, we unify three different approaches for scalable inference in structured models:
(1) randomized algorithms for MAX SAT \citep{goemans:discmath94},
(2) local consistency relaxation \citep{wainwright:book08} for discrete Markov random fields defined using Boolean logic,
and (3) reasoning about continuous information with fuzzy logic.
We show that all three approaches lead to the same convex programming objective.
We then define HL-MRFs by generalizing this unified inference objective as a
weighted sum of hinge-loss features and using them as the weighted features of graphical models.
Since HL-MRFs generalize approaches that reason about relational data with weighted logical knowledge
bases, they retain the same high level of expressivity.
As we show in Section~\ref{sec:experiments}, they are effective for modeling both discrete and continuous data.

We also introduce \emph{probabilistic soft logic} (PSL), a new probabilistic programming language
that makes HL-MRFs easy to define and use for large, relational data sets.\footnote{An open source implementation, tutorials, and data sets are available at \url{http://psl.linqs.org}.}
This idea has been explored for other classes of models,
such as Markov logic networks \citep{richardson:ml06} for discrete MRFs,
relational dependency networks \citep{neville:jmlr07} for dependency networks,
and probabilistic relational models \citep{getoor:jmlr02} for Bayesian networks.
We build on these previous approaches, as well as the connection between hinge-loss potentials and logical clauses,
to define PSL.
In addition to probabilistic rules, PSL provides syntax that enables users to easily apply many common
modeling techniques, such as domain and range constraints, blocking and canopy functions,
and aggregate variables defined over other random variables.

Our next contribution is to introduce a number of inference and learning algorithms.
First, we examine MAP inference, i.e., the problem of finding a most probable assignment to the
unobserved random variables.
MAP inference in HL-MRFs is always a convex optimization.
Although any off-the-shelf optimization toolkit could be used, such methods typically do not
leverage the sparse dependency structures common in graphical models.
We introduce a consensus-optimization approach to MAP inference for HL-MRFs, showing how the
problem can be decomposed using the alternating direction method of multipliers (ADMM) and how the
resulting subproblems can be solved analytically for hinge-loss potentials.
Our approach enables HL-MRFs to easily scale beyond the capabilities of off-the-shelf optimization software or sampling-based
inference in discrete MRFs.
We then show how to learn HL-MRFs from training data using a variety of methods: structured perceptron, maximum pseudolikelihood, and large margin estimation.
Since structured perceptron and large margin estimation rely on inference as subroutines, and maximum
pseudolikelihood estimation is efficient by design, all of these methods are highly scalable for HL-MRFs.
We evaluate them on core relational learning and structured prediction tasks, such as collective classification and link prediction.
We show that HL-MRFs offer predictive accuracy comparable to analogous discrete models
while scaling much better to large data sets.

This paper brings together and expands work on scalable models for structured data that can be either discrete,
continuous, or a mixture of both \citep{broecheler:uai10, bach:nips12, bach:uai13, bach:aistats15}.
The effectiveness of HL-MRFs and PSL has been demonstrated on many problems,
including information extraction \citep{liu:aaai16} and automatic knowledge base construction \citep{pujara:iswc13},
extracting and evaluating natural-language arguments on the Web \citep{samadi:aaai16},
high-level computer vision \citep{london:sptli13},
drug discovery \citep{fakhraei:tcbb14} and predicting drug-drug interactions \citep{sridhar:bioinformatics16},
natural language semantics \citep{beltagy:acl14, sridhar:acl15, deng:emnlp15, ebrahimi:emnlp16},
automobile-traffic modeling \citep{chen:icdm14},
recommender systems \citep{kouki:recsys15},
information retrieval \citep{alshukaili:iswc16},
and predicting attributes \citep{li:arxiv14} and trust \citep{huang:sbp13, west:tacl14} in social networks.
The ability to easily incorporate latent variables into HL-MRFs and PSL \citep{bach:icml15} has enabled further applications,
including modeling latent topics in text \citep{foulds:icml15},
and predicting student outcomes in massive open online courses (MOOCs) \citep{ramesh:aaai14, ramesh:acl15}.
Researchers have also studied how to make HL-MRFs and PSL even more scalable by developing distributed
implementations \citep{miao:bigdata13, magliacane:krr15}.
That they are already being widely applied indicates HL-MRFs and PSL address an open need in the
machine learning community.

The paper is organized as follows.
In Section~\ref{sec:logic}, we first consider models for structured prediction that are defined using logical clauses.
We unify three different approaches to scalable inference in such models, showing that they all optimize
the same convex objective.
We then generalize this objective in Section~\ref{sec:hlmrf} to define HL-MRFs.
In Section~\ref{sec:psl}, we introduce PSL, specifying the language and giving many examples of common usage.
Next we introduce a scalable message-passing algorithm for MAP inference in
Section~\ref{sec:map} and a number of learning algorithms in Section~\ref{sec:learning}, evaluating them on a range
of tasks.
Finally, in Section~\ref{sec:related}, we discuss related work.

%
%


\section{Unifying Convex Inference for Logic-Based Graphical Models}
\label{sec:logic}

In many structured domains, propositional and first-order logics are useful tools for describing
the intricate dependencies that connect the unknown variables.
However, these domains are usually noisy; dependencies among the variables do not always hold.
To address this, logical semantics can be incorporated into probability distributions to create models
that capture both the structure and the uncertainty in machine learning tasks.
One common way to do this is to use logic to define feature functions in a probabilistic model.
We focus on Markov random fields (MRFs), a popular class of probabilistic graphical models.
Informally, an MRF is a distribution that assigns probability mass using a scoring function that is a
weighted combination of feature functions called potentials.
We will use logical clauses to define these potentials.
We first define MRFs more formally to introduce necessary notation:
\begin{definition}
\label{def:mrf}
Let $\msvarset = (\msvar_1,\dots,\msvar_\nmsvar)$ be a vector of random variables and let
$\ppot = (\pot_1,\dots,\pot_\npot)$ be a vector of potentials where each potential
$\pot_\ipot(\msvarset)$ assigns configurations of the variables a real-valued score.
Also, let $\pparam = (\param_1, \dots, \param_\npot)$ be a vector of real-valued weights.
Then, a {\bf Markov random field} is a probability distribution of the form
\begin{equation}
\prob(\msvarset) \propto \exp \left(\pparam^\top\ppot(\msvarset) \right)~.
\end{equation}
\end{definition}
In an MRF, the potentials should capture how the domain behaves, assigning higher scores to more probable configurations of the variables.
If a modeler does not know how the domain behaves, the potentials should capture how it might
behave, so that a learning algorithm can find weights that lead to accurate predictions.
Logic provides an excellent formalism for defining such potentials in structured and relational domains.

We now introduce some notation to make this logic-based approach more formal.
Consider a set of logical clauses $\msclauseset = \{\msclause_1,\dots,\msclause_\nmsclause\}$, i.e., a knowledge base, where
each clause $\msclause_\imsclause \in \msclauseset$ is a disjunction of literals
and each literal is a variable $\msvar$ or its negation $\neg \msvar$ drawn from the variables
$\msvarset$ such that each variable $\msvar_\imsvar \in \msvarset$ appears at most once in $\msclause_\imsclause$.
Let $\msindicatorset_\imsclause^+$~(resp. $\msindicatorset_\imsclause^-$)~$\subset \{1,\dots,\nmsvar \}$
be the set of indices of the variables that are not negated (resp. negated) in $\msclause_\imsclause$.
Then $\msclause_\imsclause$ can be written as
\begin{equation}
\left( \bigvee_{\imsvar \in \msindicatorset_\imsclause^+} \msvar_\imsvar \right)
\bigvee
\left( \bigvee_{\imsvar \in \msindicatorset_\imsclause^-} \neg \msvar_\imsvar \right)~.
\end{equation}

Logical clauses of this form are expressive because they can be viewed equivalently as implications from conditions to consequences:
\begin{equation}
\bigwedge_{\imsvar \in \msindicatorset_\imsclause^-} \msvar_\iy
\implies
\bigvee_{\imsvar \in \msindicatorset_\imsclause^+} \msvar_\imsvar~.
\end{equation}
This ``if-then'' reasoning is intuitive and can describe many dependencies in structured data.

Assuming we have a logical knowledge base $\msclauseset$ describing a structured domain,
we can embed it in an MRF by defining each potential $\pot_\ipot$ using a corresponding clause $\msclause_\imsclause$.
If an assignment to the variables $\msvarset$ satisfies $\msclause_\imsclause$, then we let
$\pot_\ipot(\msvarset)$ equal 1, and we let it equal 0 otherwise.
For our subsequent analysis we assume $\param_\ipot \geq 0$ $(\forall \ipot = 1,\dots,\npot)$.
The resulting MRF preserves the structured dependencies described in $\msclauseset$ but enables
much more flexible modeling.
Clauses no longer must always hold, and the model can express uncertainty over different possible worlds.
The weights express how strongly the model expects each corresponding clause to hold; the higher the weight,
the more probable that it is true according to the model.

This notion of embedding weighted, logical knowledge bases in MRFs is an appealing one.
For example, Markov logic \citep{richardson:ml06} is a popular formalism that induces MRFs from weighted
first-order knowledge bases.
Given a data set, the first-order clauses are grounded using the constants in the data to create the set
of propositional clauses $\msclauseset$.
Each propositional clause has the weight of the first-order clause from which it was grounded.
In this way, a weighted, first-order knowledge base can compactly specify an entire family of MRFs
for a structured machine-learning task.

Although we now have a method for easily defining rich, structured models for a wide range of problems,
there is a new challenge: finding a most probable assignment to the variables, i.e., MAP inference, is NP-hard \citep{shimony:ai94, garey:tcs76}.
This means that (unless P=NP) our only hope for performing tractable inference is to perform it approximately.
Observe that MAP inference for an MRF defined by $\msclauseset$ is the integer linear program
\begin{equation}
\label{eq:maxsat}
\begin{aligned}
\argmax_{\msvarset \in \{0,1\}^\nmsvar} \prob(\msvarset)
&~\equiv~
\argmax_{\msvarset \in \{0,1\}^\nmsvar}~\pparam^\top \ppot(\msvarset) \\
&\equiv~
\argmax_{\msvarset \in \{0,1\}^\nmsvar}
\sum_{\msclause_\imsclause \in \msclauseset} \param_\imsclause
\min \left\{
\sum_{\imsvar \in \msindicatorset_\imsclause^+} \msvar_\imsvar
+ \sum_{\imsvar \in \msindicatorset_\imsclause^-} (1- \msvar_\imsvar)
, 1 \right\}~.
\end{aligned}
\end{equation}
While this program is intractable, it does admit convex programming relaxations.

In this section, we show how convex programming can be used to perform tractable inference in
MRFs defined by weighted knowledge bases.
We first discuss in Section~\ref{sec:maxsat} an approach developed by \citet{goemans:discmath94} that
views MAP inference as an instance of the classic MAX SAT problem and relaxes it to a convex program
from that perspective.
This approach has the advantage of providing strong guarantees on the quality of the discrete solutions it obtains.
However, it has the disadvantage that general-purpose convex programming toolkits do not scale well
to relaxed MAP inference for large graphical models \citep{yanover:jmlr06}.
In Section~\ref{sec:lcr} we then discuss a seemingly distinct approach, local consistency relaxation,
with complementary advantages and disadvantages: it offers highly scalable message-passing algorithms but comes with no quality guarantees.
We then unite these approaches by proving that they solve equivalent optimization problems with identical solutions.
Then, in Section~\ref{sec:fuzzy}, we show that the unified inference objective is also equivalent to exact
MAP inference if the knowledge base $\msclauseset$ is interpreted using
{\L}ukasiewicz logic, an infinite-valued logic for reasoning about naturally continuous quantities such
as similarity, vague or fuzzy concepts, and real-valued data.

That these three interpretations all lead to the same inference objective---whether reasoning
about discrete or continuous information---is useful.
To the best of our knowledge, we are the first to show their equivalence.
This equivalence indicates that the same modeling formalism, inference algorithms, and learning algorithms can be
used to reason scalably and accurately about both discrete and continuous information in
structured domains.
We generalize the unified inference objective in Section~\ref{sec:derivation} to define hinge-loss MRFs,
and in the rest of the paper we develop a probabilistic programming language and algorithms that
realize the goal of a scalable and accurate framework for structured data, both discrete and continuous.

\subsection{MAX SAT Relaxation}
\label{sec:maxsat}

One approach to approximating objective~(\ref{eq:maxsat}) is to use relaxation techniques developed in
the randomized algorithms community for the MAX SAT problem.
Formally, the MAX SAT problem is to find a Boolean assignment to a set of variables that maximizes the
total weight of satisfied clauses in a knowledge base composed of disjunctive clauses annotated
with nonnegative weights.
In other words, objective~(\ref{eq:maxsat}) is an instance of MAX SAT.
Randomized approximation algorithms can be constructed for MAX SAT by independently rounding each
Boolean variable $\msvar_\imsvar$ to true with probability $\msprob_\imsvar$.
Then, the expected weighted satisfaction $\msexpweight_\imsclause$ of a clause $\msclause_\imsclause$ is
\begin{equation}
\msexpweight_\imsclause
= \msweight_\imsclause \left(1 - \prod_{\imsvar \in \msindicatorset_\imsclause^+}
(1-\msprob_\imsvar) \prod_{\imsvar \in \msindicatorset_\imsclause^-}\msprob_\imsvar \right)~,
\end{equation}
also known as a (weighted) noisy-or function, and the expected total score $\msexptotalweight$ is
\begin{equation}
\label{eq:expected}
\msexptotalweight = \sum_{\msclause_\imsclause \in \msclauseset}
\msweight_\imsclause \left(1 - \prod_{\imsvar \in \msindicatorset_\imsclause^+}
(1-\msprob_\imsvar) \prod_{\imsvar \in \msindicatorset_\imsclause^-}\msprob_\imsvar \right)~.
\end{equation}
Optimizing $\msexptotalweight$ with respect to the rounding probabilities would give the exact
MAX SAT solution, so this randomized approach has not made the problem any easier yet, but
\citet{goemans:discmath94} showed how to bound $\msexptotalweight$ below with a tractable linear program.

To approximately optimize $\msexptotalweight$, associate with each Boolean variable $\msvar_\imsvar$
a corresponding continuous variable $\mssoftvar_\imsvar$ with domain $[0,1]$.
Then let $\mssoftvarset^\star$ be the optimum of the linear program
\begin{equation}
\label{eq:lpmaxsat}
\argmax_{\mssoftvarset \in [0,1]^\nmsvar}
\sum_{\msclause_\imsclause \in \msclauseset}
\msweight_\imsclause
\min \left\{ \sum_{\imsvar \in \msindicatorset_\imsclause^+} \mssoftvar_\imsvar
+ \sum_{\imsvar \in \msindicatorset_\imsclause^-}(1-\mssoftvar_\imsvar), 1 \right\}~.
\end{equation}
Observe that objectives~(\ref{eq:maxsat}) and~(\ref{eq:lpmaxsat}) are of the same form, except that
the variables are relaxed to the unit hypercube in objective~(\ref{eq:lpmaxsat}).
\citet{goemans:discmath94} proved that if $\msprob_\imsvar$ is set to $\mssoftvar_\imsvar^\star$ for all
$\imsvar$, then $\msexptotalweight \geq .632~\msilpopt$,
where $\msilpopt$ is the optimal total weight for the MAX SAT problem.
If each $\msprob_\imsvar$ is set using any function in a special class, then this lower bound improves
to a .75 approximation.
One simple example of such a function is
\begin{equation}
\label{eq:msprobfunc}
\msprob_\imsvar = \frac{1}{2} \mssoftvar_\imsvar^\star + \frac{1}{4}~.
\end{equation}
In this way, objective~(\ref{eq:lpmaxsat}) leads to an expected .75 approximation of the MAX SAT solution.

The following method of conditional probabilities \citep{alon:pmbook08} can find a single Boolean assignment that
achieves at least the expected score from a set of rounding probabilities, and therefore at least .75 of the
MAX SAT solution when objective~(\ref{eq:lpmaxsat}) and function~(\ref{eq:msprobfunc}) are used to obtain them.
Each variable $\msvar_\imsvar$ is greedily set to the value that maximizes the expected weight over
the unassigned variables, conditioned on either possible value of $\msvar_\imsvar$ and the previously
assigned variables.
This greedy maximization can be applied quickly because, in many models, variables only participate in
a small fraction of the clauses, making the change in expectation quick to compute for each variable.
Specifically, referring to the definition of $\hat{W} $ (\ref{eq:expected}), the assignment to $\msvar_\imsvar$
only needs to maximize over the clauses $\msclause_\imsclause$ in which $\msvar_\imsvar$ participates, i.e.,
$\imsvar \in \msindicatorset_\imsclause^+ \cup \msindicatorset_\imsclause^-$, which is usually a small set.

This approximation is powerful because it is a tractable linear program that comes with strong guarantees
on solution quality.
However, even though it is tractable, general-purpose convex optimization toolkits do not scale well
to large MAP problems.
In the following subsection, we unify this approximation with a complementary one developed
in the probabilistic graphical models community.

\subsection{Local Consistency Relaxation}
\label{sec:lcr}

Another approach to approximating objective~(\ref{eq:maxsat}) is to apply a relaxation developed for
Markov random fields called local consistency relaxation \citep{wainwright:book08}.
This approach starts by viewing MAP inference as an equivalent optimization over marginal
probabilities.\footnote{This treatment is for discrete MRFs. We have omitted a discussion of continuous MRFs for conciseness.}
For each $\pot_\ipot \in \ppot$, let $\ppmlpfactorvar_\ipot$ be a marginal distribution over joint
assignments $\xx_\ipot$.
For example, $\pmlpfactorvar_\ipot(\xx_\ipot)$ is the probability that the subset of variables associated with potential $\pot_\ipot$ is in a particular joint state $\xx_\ipot$.
Also, let $\x_\ipot(\iy)$ denote the setting of the variable with index $\iy$ in the state $\xx_\ipot$.

With this variational formulation, inference can be relaxed
to an optimization over the {\em first-order local polytope} $\localpolytope$.
Let $\ppmlpvar = (\pmlpvar_1, \dots, \pmlpvar_\ny)$ be a vector of probability distributions,
where $\pmlpvar_\iy(\istate)$ is the marginal probability that $\x_i$ is in state $\istate$.
The first-order local polytope is
\begin{equation}
\localpolytope \triangleq \left\{(\ppmlpfactorvar,\ppmlpvar) \geq {\boldsymbol 0}~\middle|~
\begin{array}{lr}
	\sum_{\msvarset_\imsclause| \msvar_\imsclause(\imsvar) = \istate}
		\pmlpfactorvar_\imsclause(\msvarset_\imsclause) = \pmlpvar_\imsvar(\istate)
		& \forall \iy, \ipot, \istate \\[0.5em]
	\sum_{\msvarset_\imsclause} \pmlpfactorvar_\imsclause(\msvarset_\imsclause) = 1
		& \forall \ipot \\[0.5em]
	\sum_{\istate = 0}^{\nstate_\iy-1} \pmlpvar_\iy(\istate) = 1 & \forall \iy
\end{array}
\right\},
\end{equation}
which constrains each marginal distribution $\ppmlpfactorvar_\ipot$ over joint states $\xx_\ipot$
to be consistent only with the marginal distributions $\ppmlpvar$ over individual variables that 
participate in the potential $\pot_\ipot$.

MAP inference can then be approximated with the {\em first-order local consistency relaxation}:
\begin{equation}
\label{eq:localvariational}
\argmax_{(\ppmlpfactorvar, \ppmlpvar) \in \localpolytope}~\sum_{\ipot=1}^\npot \param_\ipot
\sum_{\xx_\ipot} \pmlpfactorvar_\ipot(\xx_\ipot)~\pot_\ipot(\xx_\ipot),
\end{equation}
which is an upper bound on the true MAP objective.
Much work has focused on solving the first-order local consistency relaxation for large-scale MRFs,
which we discuss further in Section~\ref{sec:related}.
These algorithms are appealing because they are well-suited to the sparse dependency structures
common in MRFs, so they can scale to large problems.
However, in general, the solutions can be fractional, and there are no guarantees on the approximation
quality of a tractable discretization of these fractional solutions.

We show that for MRFs with potentials defined by $\msclauseset$ and nonnegative weights,
local consistency relaxation is equivalent to MAX SAT relaxation.
\begin{theorem}
\label{thm:equivalence}
For an MRF with potentials corresponding to disjunctive logical clauses and associated nonnegative
weights, the first-order local consistency relaxation of MAP inference is equivalent to the
MAX SAT relaxation of \citet{goemans:discmath94}.
Specifically, any partial optimum $\ppmlpvar^\star$ of objective~(\ref{eq:localvariational}) is an optimum
$\mssoftvarset^\star$ of objective~(\ref{eq:lpmaxsat}), and vice versa.
\end{theorem}
We prove Theorem~\ref{thm:equivalence} in Appendix~\ref{app:equivalence}.
Our proof analyzes the local consistency relaxation to derive an equivalent, more compact optimization over
 only the variable pseudomarginals $\ppmlpvar$ that is identical to the MAX SAT relaxation.
Theorem~\ref{thm:equivalence}  is significant because it shows that the rounding guarantees
of MAX SAT relaxation also apply to local consistency relaxation, and the scalable message-passing
algorithms developed for local consistency relaxation also apply to MAX SAT relaxation.

\subsection{{\L}ukasiewicz Logic}
\label{sec:fuzzy}

The previous two subsections showed that the same convex program can approximate MAP
inference in discrete, logic-based models, whether viewed from the perspective of randomized algorithms or variational methods.
In this subsection, we show that this convex program can also be used to reason about naturally continuous information,
such as similarity, vague or fuzzy concepts, and real-valued data.
Instead of interpreting the clauses $\msclauseset$ using Boolean logic, we can interpret
them using {\L}ukasiewicz logic \citep{klir:book95}, which extends  Boolean logic to infinite-valued logic in
which the propositions $\msvarset$ can take truth values in the continuous interval $[0,1]$.
Extending truth values to a continuous domain enables them to represent concepts that are vague, in
the sense that they are often neither completely true nor completely false.
For example, the propositions that a sensor value is high, two entities are similar, or a protein is highly
expressed can all be captured in a more nuanced manner in {\L}ukasiewicz logic.
We can also use the now continuous valued $\msvarset$ to represent quantities that are naturally
continuous (scaled to [0,1]), such as actual sensor values, similarity scores, and protein expression levels.
The ability to reason about continuous values is valuable, as many important applications are not
entirely discrete.

The extension to continuous values requires a corresponding extended interpretation of the logical operators
$\wedge$ (conjunction), $\vee$ (disjunction), and $\neg$ (negation).
The {\L}ukasiewicz t-norm and t-co-norm are $\wedge$ and $\vee$ operators that correspond to the Boolean logic operators
for integer inputs (along with the negation operator $\neg$):
\begin{align}
\msvar_1 \wedge \msvar_2 &= \max \left\{ \msvar_1 + \msvar_2 - 1, 0 \right\} \\
\msvar_1 \vee \msvar_2 &= \min \left\{ \msvar_1 + \msvar_2, 1 \right\} \\
\neg \msvar &= 1 - \msvar~.
\end{align}
The analogous MAX SAT problem for {\L}ukasiewicz logic is therefore
\begin{equation}
\argmax_{\msvarset \in [0,1]^\nmsvar} \sum_{\msclause_\imsclause \in \msclauseset} \msweight_\imsclause
\min \left\{ \sum_{\imsvar \in \msindicatorset_\imsclause^+}\msvar_\imsvar + \sum_{\imsvar \in \msindicatorset_\imsclause^-}(1-\msvar_\imsvar), 1 \right\}~,
\end{equation}
which is identical in form to the relaxed MAX SAT objective~(\ref{eq:lpmaxsat}).
Therefore, if an MRF is defined over continuous variables with domain $[0,1]^\ny$ and the logical
knowledge base $\msclauseset$ defining the potentials is interpreted using {\L}ukasiewicz logic, then
\emph{exact} MAP inference is identical to finding the optimum using the unified, relaxed inference
objective derived for Boolean logic in the previous two subsections.
This result shows the equivalence of all three approaches: MAX SAT relaxation, local consistency relaxation,
and MAX SAT using {\L}ukasiewicz logic.


\section{Hinge-Loss Markov Random Fields}
\label{sec:hlmrf}

We have shown that a specific family of convex programs can be used to reason scalably and accurately about both
discrete and continuous information.
In this section, we generalize this family to define \emph{hinge-loss Markov random fields}
(HL-MRFs), a new kind of probabilistic graphical model.
HL-MRFs retain the convexity and expressivity of convex programs discussed in Section~\ref{sec:logic}, and additionally support an even richer space of dependencies.

To begin, we define HL-MRFs as density functions over continuous variables $\yy = (\y_1,\dots,\y_\ny)$
with joint domain $[0,1]^\ny$.
These variables have different possible interpretations depending on the application.
Since we are generalizing the interpretations explored in Section~\ref{sec:logic}, HL-MRF MAP states can be viewed
as rounding probabilities or pseudomarginals, or they can represent naturally continuous information.
More generally, they can be viewed simply as degrees of belief, confidences, or rankings of possible states;
and they can describe discrete, continuous, or mixed domains.
The application domain typically determines which interpretation is most appropriate.
The formalisms and algorithms described in the rest of this paper are general with respect to such interpretations.

\subsection{Generalized Inference Objective}
\label{sec:derivation}

To define HL-MRFs, we will first generalize the unified inference objective of Section~\ref{sec:logic} in several ways,
which we first restate in terms of the HL-MRF variables $\yy$:
\begin{equation}
\argmax_{\yy \in [0,1]^\nmsvar}
\sum_{\msclause_\imsclause \in \msclauseset}
\msweight_\imsclause
\min \left\{ \sum_{\imsvar \in \msindicatorset_\imsclause^+} \y_\imsvar
+ \sum_{\imsvar \in \msindicatorset_\imsclause^-}(1-\y_\imsvar), 1 \right\}~.
\end{equation}
For now, we are still assuming that the objective terms are defined using a weighted knowledge base $\msclauseset$,
but we will quickly drop this requirement.
To do so, we examine one term in isolation.
Observe that the maximum value of any unweighted term is 1, which is achieved
when a linear function of the variables is at least 1.
We say that the term is \emph{satisfied} whenever this occurs.
When a term is unsatisfied, we can refer to its \emph{distance to satisfaction}, which is how far it is from achieving its maximum value.
Also observe that we can rewrite the optimization explicitly in terms of distances to satisfaction:
\begin{equation}
\argmin_{\yy \in [0,1]^\nmsvar}
\sum_{\msclause_\imsclause \in \msclauseset}
\msweight_\imsclause
\max \left\{ 1 - \sum_{\imsvar \in \msindicatorset_\imsclause^+} \y_\imsvar
- \sum_{\imsvar \in \msindicatorset_\imsclause^-}(1-\y_\imsvar), 0 \right\}~,
\end{equation}
so that the objective is equivalently to minimize the total weighted distance to satisfaction.
Each unweighted objective term now measures how far the linear constraint
\begin{equation}
1 - \sum_{\imsvar \in \msindicatorset_\imsclause^+} \y_\imsvar
- \sum_{\imsvar \in \msindicatorset_\imsclause^-}(1-\y_\imsvar) \leq 0
\end{equation}
is from being satisfied.

\subsubsection{Relaxed Linear Constraints}

With this view of each term as a relaxed linear constraint, we can easily generalize them to arbitrary linear constraints.
We no longer require that the inference objective be defined using only logical clauses, and instead each term can
be defined using any function $\linfun_\ipot(\yy)$ that is linear in $\yy$.
These functions can capture more general dependencies, such as beliefs about the range of values a variable can take and arithmetic relationships among variables.

The new inference objective is
\begin{equation}
\argmin_{\yy \in [0,1]^\nmsvar}
\sum_{\ipot = 1}^\npot
\msweight_\imsclause
\max \left\{ \linfun_\ipot(\yy), 0 \right\}~.
\end{equation}
In this form, each term represents the distance to satisfaction of a linear constraint $\linfun_\ipot(\yy) \leq 0$.
That constraint could be defined using logical clauses as discussed above, or it could be defined using
other knowledge about the domain.
The weight $\param_\ipot$ indicates how important it is to satisfy a constraint relative to others by scaling
the distance to satisfaction.
The higher the weight, the more distance to satisfaction is penalized.
Additionally, two relaxed inequality constraints, $\linfun_\ipot(\yy) \leq 0$ and
$-\linfun_\ipot(\yy) \leq 0$, can be combined to represent a relaxed equality constraint
$\linfun_\ipot(\yy) = 0$.

\subsubsection{Hard Linear Constraints}

Now that our inference objective admits arbitrary relaxed linear constraints, it is natural to also allow hard
constraints that must be satisfied at all times.
Hard constraints are important modeling tools.
They enable groups of variables to represent mutually exclusive possibilities, such as a multinomial or categorical variable, and functional or partial functional relationships.
Hard  constraints can also represent background knowledge about the domain, restricting
the domain to regions that are feasible in the real world.
Additionally, they can encode more complex model components such as defining a random variable as
an aggregate over other unobserved variables, which we discuss further in Section~\ref{sec:aggregates}.

We can think of including hard constraints as allowing a weight $\param_\ipot$ to take an infinite value.
Again, two inequality constraints can be combined to represent an equality constraint.
However, when we introduce an inference algorithm for HL-MRFs in Section~\ref{sec:map}, it will be
useful to treat hard constraints separately from relaxed ones, and further, treat hard inequality
constraints separately from hard equality constraints.
Therefore, in the definition of HL-MRFs, we will define these three components separately.

\subsubsection{Generalized Hinge-Loss Functions}

The objective terms measuring each constraint's distance to satisfaction are hinge losses.
There is a flat region, on which the distance to satisfaction is 0, and an angled region, on which the
distance to satisfaction grows linearly away from the hyperplane $\linfun_\ipot(\yy) = 0$.
This loss function is useful---as we discuss in the previous section, it is a bound
on the expected loss in the discrete setting, among other things---but it is not appropriate
for all modeling situations.

A piecewise-linear loss function makes MAP inference ``winner take all,'' in the sense that it is preferable
to fully satisfy the most highly weighted objective terms completely before reducing the
distance to satisfaction of terms with lower weights.
For example, consider the following optimization problem:
\begin{equation}
\argmin_{\y_1 \in [0,1]}~\param_1 \max \left\{ \y_1, 0 \right\} + \param_2 \max \left\{ 1 - \y_1, 0 \right\}~.
\end{equation}
If $\param_1 > \param_2 \geq 0$, then the optimizer is $\y_1 = 0$ because the term that prefers $\y_1 = 0$ overrules
the term that prefers $\y_1 = 1$.
The result does not indicate any ambiguity or uncertainty, but if the two objective terms
are potentials in a probabilistic model, it is sometimes preferable that the result
reflect the conflicting preferences.
We can change the inference problem so that it smoothly trades off satisfying conflicting
objective terms by squaring the hinge losses.
Observe that in the modified problem
\begin{equation}
\argmin_{\y_1 \in [0,1]}~\param_1 \left( \max \left\{ \y_1, 0 \right\} \right)^2
+ \param_2 \left( \max \left\{ 1 - \y_1, 0 \right\} \right)^2
\end{equation}
the optimizer is now $\y_1 = \frac{\param_2}{\param_1 + \param_2}$, reflecting the relative influence of the two loss functions.

Another advantage of squared hinge-loss functions is that they can behave more intuitively in the
presence of hard constraints.
Consider the problem
\begin{equation}
\begin{aligned}
& \argmin_{(\y_1,\y_2) \in [0,1]^2} &\max \left\{ 0.9 - \y_1, 0 \right\} + \max \left\{ 0.6 - \y_2, 0 \right\} \\
&\phantom{a} & \\
&\text{such that} &\y_1 + \y_2 \leq 1~.
\end{aligned}
\end{equation}
The first term prefers $\y_1 \geq 0.9$, the second term prefers $\y_2 \geq 0.6$, and the constraint
requires that $\y_1$ and $\y_2$ are mutually exclusive.
Such problems are very common and arise when conflicting evidence of different strengths
support two mutually exclusive possibilities.
The evidence values 0.9 and 0.6 could come from many sources, including base models trained
to make independent predictions on individual random variables, domain-specialized similarity
functions, or sensor readings.
For this problem, any solution $\y_1 \in [0.4, 0.9]$ and $\y_2 = 1 - \y_1$ is an optimizer.
This solution set includes counterintuitive optimizers like $\y_1 = 0.4$ and $\y_2 = 0.6$, even though
the evidence supporting $\y_1$ is stronger.
Again, squared hinge losses ensure the optimizers better reflect the relative strength of evidence.
For the problem
\begin{equation}
\begin{aligned}
& \argmin_{(\y_1,\y_2) \in [0,1]^2}
&\left( \max \left\{ 0.9 - \y_1, 0 \right\} \right)^2 + \left( \max \left\{ 0.6 - \y_2, 0 \right\} \right)^2 \\
&\phantom{a} & \\
&\text{such that}
&\y_1 + \y_2 \leq 1~,
\end{aligned}
\end{equation}
the only optimizer is $\y_1 = 0.65$ and $\y_2 = 0.35$, which is a more informative solution.

We therefore complete our generalized inference objective by allowing either hinge-loss or
squared hinge-loss functions.
Users of HL-MRFs have the choice of either one for each potential, depending on which is appropriate for their task.

\subsection{Definition}

We can now formally state the full definition of HL-MRFs.
They are defined so that a MAP state is a solution to the generalized inference objective proposed in the previous subsection.
We state the definition in a conditional form for later convenience, but this definition is fully general since
the vector of conditioning variables may be empty.

\begin{definition}
\label{def:energy}
Let $\yy = (\y_1, \dots, \y_{\ny})$ be a vector of $\ny$ variables
and $\xx = (\x_1, \dots, \x_{\nx})$ a vector of $\nx$ variables
with joint domain $\Dom = [0,1]^{\ny+\nx}$.
Let ${\boldsymbol \pot} = (\pot_1, \dots, \pot_\npot)$ be a vector of $\npot$ continuous potentials of the form
\begin{equation}
\pot_\ipot(\yy, \xx) = \left(\max{\{\linfun_\ipot(\yy, \xx), 0\}}\right)^{\p_{\ipot}}
\end{equation}
where $\linfun_\ipot$ is a linear function of $\yy$ and $\xx$ and $\p_{\ipot} \in \{1,2\}$.
Let ${\boldsymbol \constr} = (\constr_1, \ldots, \constr_\nconstr)$ be a vector of $\nconstr$ linear constraint functions associated with index sets denoting equality constraints $\Equality$ and inequality constraints $\Inequality$, which define the feasible set
\begin{equation}
\FeasibleDom = \left\{(\yy, \xx) \in \Dom \middle| 
\begin{array}{c}
\constr_\iconstr(\yy, \xx) = 0, \forall \iconstr \in \Equality\\
\constr_\iconstr(\yy, \xx) \leq 0, \forall \iconstr \in \Inequality
\end{array}
\right\}~.
\end{equation}
For $(\yy, \xx) \in \Dom$, given a vector of $\npot$ nonnegative free parameters, i.e., weights,
$\pparam = (\param_1, \dots, \param_\npot)$,
a {\bf constrained hinge-loss energy function} $\energy$ is defined as
\begin{equation}
\energy(\yy,\xx) = \sum_{\ipot=1}^\npot \param_\ipot \phi_\ipot(\yy, \xx)~.
\end{equation}
\end{definition}

We now define HL-MRFs by placing a probability density over the inputs to a constrained hinge-loss energy function.
Note that we negate the hinge-loss energy function so that states with lower energy are more probable, in contrast
with Definition~\ref{def:mrf}.
This change is made for later notational convenience.
\begin{definition}
\label{def:hl-mrf}
A {\bf hinge-loss Markov random field} $\prob$ over random variables $\yy$
and conditioned on random variables $\xx$ is a probability density defined as follows:
if $(\yy, \xx) \notin \FeasibleDom$, then $\prob(\yy | \xx) = 0$;
if $(\yy, \xx) \in \FeasibleDom$, then
\begin{equation}
\prob(\yy | \xx) = \frac{1}{\partition(\pparam, \xx)} \exp \left( - \energy(\yy,\xx) \right)
\end{equation}
where
\begin{equation}
\partition(\pparam, \xx) = \int_{\yy | (\yy,\xx) \in \FeasibleDom} \exp \left( - \energy(\yy,\xx) \right)~d\yy~.
\end{equation}
\end{definition}

In the rest of this paper, we will explore how to use HL-MRFs to solve a wide range of structured machine
learning problems.
We first introduce a probabilistic programming language that makes HL-MRFs easy to define
for large, rich domains.


\section{Probabilistic Soft Logic}
\label{sec:psl}

In this section we introduce a general-purpose probabilistic programming language,
{\em probabilistic soft logic} (PSL).
PSL allows HL-MRFs to be easily applied to a broad range of structured machine learning problems
by defining \emph{templates} for potentials and constraints.
In models for structured data, there are very often repeated patterns of probabilistic dependencies.
A few of the many examples include the strength of ties between similar people in social networks,
the preference for triadic closure when predicting transitive relationships,
and the ``exactly one active'' constraints on functional relationships.
Often, to make graphical models both easy to define and able to generalize across different
data sets, these repeated dependencies are defined using templates.
Each template defines an abstract dependency, such as the form of a potential function or
constraint, along with any necessary
parameters, such as the weight of the potential, each of which has a single value across all dependencies defined by that template.
Given input data, an undirected graphical model is constructed from a set of templates by first identifying
the random
variables in the data and then ``grounding out'' each template by introducing a potential or constraint into
the graphical model for each subset of random variables to which the template applies.

A PSL program is written in a declarative, first-order syntax and defines a class of HL-MRFs that are parameterized by
the input data.
PSL provides a natural interface to represent hinge-loss potential templates using two types of rules:
logical rules and arithmetic rules.
Logical rules are based on the mapping from logical clauses to hinge-loss potentials introduced in Section~\ref{sec:logic}.
Arithmetic rules provide additional syntax for defining an even wider range of hinge-loss potentials and hard constraints.

\subsection{Definition}
\label{sec:def_psl}

In this subsection we define PSL.
Our definition covers the essential functionality that should be supported by all implementations,
but many extensions are possible.
The PSL syntax we describe can capture a wide range of HL-MRFs, but new settings and scenarios
could motivate the development of additional syntax to make the construction of different kinds of
HL-MRFs more convenient.

\subsubsection{Preliminaries}
We begin with a high-level definition of PSL programs.
\begin{definition}
A {\bf PSL program} is a set of rules, each of which is a template for hinge-loss potentials or hard linear constraints.
When grounded over a base of ground atoms, a PSL program induces a HL-MRF conditioned on any
specified observations.
\end{definition}
In the PSL syntax, many components are named using {\em identifiers}, which are strings that begin
with a letter (from the set $\left\{\texttt{A},\dots,\texttt{Z}, \texttt{a}, \dots, \texttt{z} \right\}$),
followed by zero or more letters, numeric digits, or underscores.

PSL programs are grounded out over data, so the universe over which to ground must be defined.
\begin{definition}
A {\bf constant} is a string that denotes an element in the universe over which a PSL program is grounded.
\end{definition}
Constants are the elements in a universe of discourse.
They can be entities or attributes.
For example, the constant \texttt{"person1"} can denote a person, the constant \texttt{"Adam"}
can denote a person's name, and the constant \texttt{"30"} can denote a person's age.
In PSL programs, constants are written as strings in double or single quotes.
Constants use backslashes as escape characters, so they can be used to encode quotes within constants.
It is assumed that constants are unambiguous, i.e., different constants refer to different
entities and attributes.\footnote{Note that ambiguous references to underlying entities can be modeled
by using different constants for different references and representing whether they refer to the same
underlying entity as a predicate.}
Groups of constants can be represented using variables.
\begin{definition}
A {\bf variable} is an identifier for which constants can be substituted.
\end{definition}
Variables and constants are the arguments to logical predicates.
Together, they are generically referred to as terms.
\begin{definition}
A {\bf term} is either a constant or a variable.
\end{definition}
Terms are connected by relationships called predicates.
\begin{definition}
A {\bf predicate} is a relation defined by a unique identifier and a positive integer called its arity,
which denotes the number of terms it accepts as arguments.
Every predicate in a PSL program must have a unique identifier as its name.
\end{definition}
We refer to a predicate using its identifier and arity appended with a slash.
For example, the predicate \texttt{Friends/2} is a binary predicate,
i.e., taking two arguments, which represents whether two constants are friends.
As another example, the predicate \texttt{Name/2} can relate a person to the string that is
that person's name.
As a third example, the predicate \texttt{EnrolledInClass/3} can relate two entities, a student and professor,
with an additional attribute, the subject of the class.

Predicates and terms are combined to create atoms.
\begin{definition}
An {\bf atom} is a predicate combined with a sequence of terms of length equal to
the predicate's arity.
This sequence is called the atom's arguments.
An atom with only constants for arguments is called a ground atom.
\end{definition}
Ground atoms are the basic units of reasoning in PSL.
Each represents an unknown or observation of interest and can take any value in $[0,1]$.
For example, the ground atom \texttt{Friends("person1", "person2")}
represents whether \texttt{"person1"} and \texttt{"person2"} are friends.
Atoms that are not ground are placeholders for sets of ground atoms.
For example, the atom \texttt{Friends(X, Y)} stands for all ground atoms that can be obtained
by substituting constants for variables \texttt{X} and \texttt{Y}.

\subsubsection{Inputs}

As we have already stated, PSL defines templates for hinge-loss potentials and hard linear constraints that are
grounded out over a data set to induce a HL-MRF.
We now describe how that data set is represented and provided as the inputs to a PSL program.
The first inputs are two sets of predicates: a set $\closedpredicates$ of \emph{closed}
predicates, the atoms of which are completely observed, and a set $\openpredicates$ of \emph{open}
predicates, the atoms of which may be unobserved.
The third input is the \emph{base} $\base$, which is the set of all ground atoms under consideration.
All atoms in $\base$ must have a predicate in either $\closedpredicates$ or $\openpredicates$.
These are the atoms that can be substituted into the rules and constraints of a PSL program,
and each will later be associated with a HL-MRF random variable with domain $[0,1]$.
The final input is a function $\observations : \base \to [0,1] \cup \{ \unobserved \}$ that maps
the ground atoms in the base to either an observed value in $[0,1]$ or a symbol $\unobserved$
indicating that it is unobserved.
The function $\observations$ is only valid if all atoms with a predicate in $\closedpredicates$
are mapped to a $[0,1]$ value.
Note that this definition makes the sets $\closedpredicates$ and $\openpredicates$ redundant in a sense,
since they can be derived from $\base$ and $\observations$, but it will be convenient later
to have $\closedpredicates$ and $\openpredicates$ explicitly defined.

Ultimately, the method for specifying PSL's inputs is implementation-specific,
since different choices make it more or less convenient for different scenarios.
In this paper, we will assume that $\closedpredicates$, $\openpredicates$, $\base$, and
$\observations$ exist, and we remain agnostic about how they were specified.
However, to make this aspect of using PSL more concrete, we will describe one possible
method for defining them here.

Our example method for specifying PSL's inputs is text-based.
The first section of the text input is a definition of the constants in the universe, which
are grouped into types.
An example universe definition follows.
\begin{align*}
&\texttt{Person = \{"alexis", "bob", "claudia", "david"\}} \\
&\texttt{Professor = \{"alexis", "bob"\}} \\
&\texttt{Student = \{"claudia", "david"\}} \\
&\texttt{Subject = \{"computer science", "statistics"\}}
\end{align*}
This universe includes six constants, four with two types (\texttt{"alexis"}, \texttt{"bob"},
\texttt{"claudia"}, and \texttt{"david"}) and two with one type (\texttt{"computer science"} and
\texttt{"statistics"}).

The next section of input is the definition of predicates.
Each predicate includes the types of constants it takes as arguments and whether it is
closed.
For example, we can define predicates for an advisor-student relationship prediction task as follows:
\begin{align*}
&\texttt{Advises(Professor, Student)} \\
&\texttt{Department(Person, Subject) (closed)} \\
&\texttt{EnrolledInClass(Student, Subject, Professor) (closed)}
\end{align*}
In this case, there is one open predicate (\texttt{Advises}) and two closed predicates
(\texttt{Department} and \texttt{EnrolledInClass}).

The final section of input is any associated observations.
They can be specified in a list, for example:
\begin{align*}
&\texttt{Advises("alexis", "david") = 1} \\
&\texttt{Department("alexis", "computer science") = 1} \\
&\texttt{Department("bob", "computer science") = 1} \\
&\texttt{Department("claudia", "statistics") = 1} \\
&\texttt{Department("david", "statistics") = 1} 
\end{align*}
In addition, values for atoms with the \texttt{EnrolledInClass} predicate could also be specified.
If a ground atom does not have a specified value, it will have a default observed value of 0 if its
predicate is closed or remain unobserved if its predicate is open.

We now describe how this text input is processed into the formal inputs $\closedpredicates$,
$\openpredicates$, $\base$, and $\observations$.
First, each predicate is added to either $\closedpredicates$ or $\openpredicates$ based on whether it
is annotated with the \texttt{(closed)} tag.
Then, for each predicate in $\closedpredicates$ or $\openpredicates$, ground atoms of that predicate are
added to $\base$ with each sequence of constants as arguments that can be created by selecting a
constant of each of the predicate's argument types.
For example, assume that the input file contains a single predicate definition
\[
\texttt{Category(Document, Cat\_Name)}
\]
where the universe is $\texttt{Document} = \{ \texttt{"d1"}, \texttt{"d2"} \}$ and
$\texttt{Cat\_Name} = \{ \texttt{"politics"}, \texttt{"sports"} \}$.
Then,
\begin{equation}
\base = \left\{
\begin{array}{l}
\texttt{Category("d1", "politics")}, \\
\texttt{Category("d1", "sports")}, \\
\texttt{Category("d2", "politics")}, \\
\texttt{Category("d2", "sports")}
\end{array}
\right\}~.
\end{equation}
Finally, we define the function $\observations$.
Any atom in the explicit list of observations is mapped to the given value.
Then, any remaining atoms in $\base$ with a predicate in $\closedpredicates$ are mapped to 0,
and any with a predicate in $\openpredicates$ are mapped to $\unobserved$.

Before moving on, we also note that PSL implementations can support predicates and atoms that
are defined functionally.
Such predicates can be thought of as a type of closed predicate.
Their observed values are defined as a function of their arguments.
One of the most common examples is inequality, atoms of which can be represented with the
shorthand infix operator \texttt{!=}.
For example, the following atom has a value of 1 when two variables \texttt{A} and \texttt{B}
are replaced with different constants and 0 when replaced with the same constant.
\[
\texttt{A != B}
\]
Such functionally defined predicates can be implemented without requiring their values over all
arguments to be specified by the user.

\subsubsection{Rules and Grounding}

Before introducing the syntax and semantics of specific PSL rules, we define the grounding
procedure that induces HL-MRFs in general.
Given the inputs $\closedpredicates$, $\openpredicates$, $\base$, and $\observations$,
PSL induces a HL-MRF $\prob(\yy|\xx)$ as follows.
First, each ground atom $\groundatom \in \base$ is associated with a random variable with domain $[0,1]$.
If $\observations(\groundatom) = \unobserved$, then the variable is included in the free variables $\yy$, 
and otherwise it is included in the observations $\xx$ with a value of $\observations(\groundatom)$.

With the variables in the distribution defined, each rule in the PSL program is applied to the inputs
and produces hinge-loss potentials or hard linear constraints, which are added to the HL-MRF.
In the rest of this subsection, we describe two kinds of PSL rules: logical rules and arithmetic rules.

\subsubsection{Logical Rules}

The first kind of PSL rule is a logical rule, which is made up of literals.
\begin{definition}
A {\bf literal} is an atom or a negated atom.
\end{definition}
In PSL, the prefix operator \texttt{!} or \texttt{\textasciitilde} is used for negation.
A negated atom has a value of one minus the value of the unmodified atom.
For example, if \texttt{Friends("person1", "person2")} has a value of 0.7, then
\texttt{!Friends("person1", "person2")} has a value of 0.3.
\begin{definition}
A {\bf logical rule} is a disjunctive clause of literals.
Logical rules are either weighted or unweighted.
If a logical rule is weighted, it is annotated with a nonnegative weight and optionally a power of two.
\end{definition}
Logical rules express logical dependencies in the model.
As in Boolean logic, the negation, disjunction (written as \texttt{\pslor} or \texttt{\textbar}),
and conjunction (written as \texttt{\psland} or \texttt{\&}) operators obey De Morgan's Laws.
Also, an implication (written as \texttt{\pslimpliesright} or \texttt{\pslimpliesleft})
can be rewritten as the negation of the body disjuncted with the head.
For example
\begin{align*}
&\texttt{P1(A, B) \psland\ P2(A, B) \pslimpliesright\ P3(A, B) \pslor\ P4(A, B)} \\
\equiv~ &\texttt{!(P1(A, B) \psland\ P2(A, B)) \pslor\ P3(A, B) \pslor\ P4(A, B)} \\
\equiv~ &\texttt{!P1(A, B) \pslor\ !P2(A, B) \pslor\ P3(A, B) \pslor\ P4(A, B)}
\end{align*}
Therefore, any formula written as an implication with (1) a literal or conjunction of literals in the body
and (2) a literal or disjunction of literals in the head is also a valid logical rule, because it is equivalent to a disjunctive clause.

There are two kinds of logical rules: weighted or unweighted.
A weighted logical rule is a template for a hinge-loss potential that penalizes how far the rule is from being
satisfied.
A weighted logical rule begins with a nonnegative weight and optionally ends with an exponent of two
(\texttt{\pslcaret2}).
For example, the weighted logical rule
\[
\texttt{1 : Advisor(Prof, S) \psland\ Department(Prof, Sub) \pslimpliesright\ Department(S, Sub)}
\]
has a weight of 1 and induces potentials propagating department membership from advisors to advisees.
An unweighted logical rule is a template for a hard linear constraint that requires that the rule always be satisfied.
For example, the unweighted logical rule
\[
\texttt{Friends(X, Y) \psland\ Friends(Y, Z) \pslimpliesright\ Friends(X, Z) .}
\]
induces hard linear constraints enforcing the transitivity of the \texttt{Friends/2} predicate.
Note the period (\texttt{.}) that is used to emphasize that this rule is always enforced and 
disambiguate it from weighted rules.

A logical rule is grounded out by performing all distinct substitutions from variables to constants such that the resulting ground atoms are in the base $\base$.
This procedure produces a set of \emph{ground rules}, which are rules containing only ground atoms.
Each ground rule will then be interpreted as either a potential or hard constraint in the induced HL-MRF.
For notational convenience, we assume without loss of generality that all the random variables are
unobserved, i.e., $\observations(\groundatom) = \unobserved, \forall \groundatom \in \base$.
If the input data contain any observations, the following description still applies, except that
some free variables will be replaced with observations from $\xx$.
The first step in interpreting a ground rule is to map its disjunctive clause to a linear constraint.
This mapping is based on the unified inference objective derived in Section~\ref{sec:logic}.
Any ground PSL rule is a disjunction of literals, some of which are negated.
Let $\msindicatorset^+$ be the set of indices of the variables that correspond to atoms
that are not negated in the ground rule, when expressed as a disjunctive clause,
and, likewise, let $\msindicatorset^-$ be the indices of the variables corresponding to atoms
that are negated.
Then, the clause is mapped to the inequality
\begin{equation}
1 - \sum_{\iy \in \msindicatorset^+} \y_\iy
- \sum_{\iy \in \msindicatorset^-} (1 - \y_\iy)
\leq 0~.
\end{equation}
If the logical rule that templated the ground rule is weighted with a weight of $\param$ and is
\emph{not} annotated with \texttt{\textasciicircum2}, then the potential
\begin{equation}
\pot(\yy, \xx) = \max{\left\{
1- \sum_{\iy \in \msindicatorset^+} \y_\iy
- \sum_{\iy \in \msindicatorset^-} (1 - \y_\iy), 0 \right\}}
\end{equation}
is added to the HL-MRF with a parameter of $\param$.
If the rule is weighted with a weight $\param$ and annotated with \texttt{\textasciicircum2}, then the potential
\begin{equation}
\pot(\yy, \xx) = \left(\max{\left\{
1 - \sum_{\iy \in \msindicatorset^+} \y_\iy
- \sum_{\iy \in \msindicatorset^-} (1 - \y_\iy), 0 \right\}}
\right)^2
\end{equation}
is added to the HL-MRF with a parameter of $\param$.
If the rule is unweighted, then the function
\begin{equation}
\constr(\yy, \xx) = 1 - \sum_{\iy \in \msindicatorset^+} \y_\iy
- \sum_{\iy \in \msindicatorset^-} (1 - \y_\iy)
\end{equation}
is added to the set of constraint functions and its index is included in the set $\Inequality$ to define
a hard inequality constraint $\constr(\yy, \xx) \leq 0$.

As an example of the grounding process, consider the following logical rule.
As part of a program for link prediction, it is often helpful to model the transitivity of a relationship.
\[
\texttt{3 : Friends(A, B) \psland\ Friends(B, C) -> Friends(C, A) \pslcaret2}
\]
Imagine that the input data are $\closedpredicates = \{ \}$, $\openpredicates = \{ \texttt{Friends/2} \}$,
\begin{equation}
\base = \left\{
\begin{array}{l}
\texttt{Friends("p1", "p2")}, \\
\texttt{Friends("p1", "p3")}, \\
\texttt{Friends("p2", "p1")}, \\
\texttt{Friends("p2", "p3")}, \\
\texttt{Friends("p3", "p1")}, \\
\texttt{Friends("p3", "p2")}
\end{array}
\right\}~,
\end{equation}
and $\observations(\groundatom) = \unobserved, \forall \groundatom \in \base$.
Then, the rule will induce six ground rules.
One such ground rule is
\[
\texttt{3 : Friends("p1", "p2") \psland\ Friends("p2", "p3") \pslimpliesright\ Friends("p3", "p1") \pslcaret2}
\]
which is equivalent to the following.
\[
\texttt{3 : !Friends("p1", "p2") \pslor\ !Friends("p2", "p3") \pslor\ Friends("p3", "p1") \pslcaret2}
\]
If the atoms \texttt{Friends("p1", "p2")}, \texttt{Friends("p2", "p3")}, and \texttt{Friends("p3", "p1")}
correspond to the random variables $\y_1$, $\y_2$, and $\y_3$, respectively, then this ground
rule is interpreted as the weighted hinge-loss potential
\begin{equation}
3 ~ \left( \max\{\y_1 + \y_2 - \y_3 - 1, 0\} \right)^2~.
\end{equation}
Since the grounding process uses the mapping from Section~\ref{sec:logic},
logical rules can be used to reason accurately and efficiently about both discrete and continuous information.
They are a convenient method for constructing HL-MRFs with the unified inference objective for weighted
logical knowledge bases as their MAP inference objective.
They also allow the user to seamlessly incorporate some of the additional features of HL-MRFs,
such as squared potentials and hard constraints.
Next, we introduce an even more flexible class of PSL rules.

\subsubsection{Arithmetic Rules}
\label{sec:psl_arithmetic}

Arithmetic rules in PSL are more general templates for hinge-loss potentials and hard linear constraints.
Like logical rules, they come in weighted and unweighted variants, but instead of using logical operators
they use arithmetic operators.
In general, an arithmetic rule relates two linear combinations of atoms with an inequality or an equality.
A simple example enforces the mutual exclusivity of liberal and conservative ideologies.
\[
\texttt{Liberal(P) + Conservative(P) = 1 .}
\]
Just like logical rules, arithmetic rules are grounded out by performing all possible substitutions of constants for variables to make ground atoms in the base $\base$.
In this example, each substitution for \texttt{Liberal(P)} and \texttt{Conservative(P)}
is constrained to sum to 1.
Since the rule is unweighted and arithmetic, it defines a hard constraint $\constr(\yy, \xx)$ and
its index will be included in $\Equality$ because it is an equality constraint.

To make arithmetic rules more flexible and easy to use, we define some additional syntax.
The first is a generalized definition of atoms that can be substituted with sums of ground atoms, rather
than just a single atom.
\begin{definition}
A {\bf summation atom} is an atom that takes terms and/or sum variables as arguments.
A summation atom represents the summations of ground atoms that can be obtained by substituting
individual constants for variables and summing over all possible constants for sum variables.
\end{definition}
A sum variable is represented by prepending a plus symbol (\texttt{+}) to a variable.
For example, the summation atom
\[
\texttt{Friends(P, +F)}
\]
is a placeholder for the sum of all ground atoms with predicate \texttt{Friends/2} in $\base$ that share a first argument.
Note that sum variables can be used at most once in a rule, i.e., each sum variable in a rule must
have a unique identifier.
Summation atoms are useful because they can describe dependencies without needing to
specify the number of atoms that can participate.
For example, the arithmetic rule
\[
\texttt{Label(X, +L) = 1 .}
\]
says that labels for each constant substituted for \texttt{X} should sum to one, without needing
to specify how many possible labels there are.

The substitutions for sum variables can be restricted using logical clauses as filters.
\begin{definition}
A {\bf filter clause} is a logical clause defined for a sum variable in an arithmetic rule.
The logical clause only contains atoms (1) with predicates that appear in $\closedpredicates$ and (2) that only take as arguments
(a) constants, (b) variables that appear in the arithmetic rule, and (c) the sum variable for which it is defined.
\end{definition}
Filter clauses restrict the substitutions for a sum variable in the corresponding arithmetic rule by
only including substitutions for which the clause evaluates to true.
The filters are evaluated using Boolean logic.
Each ground atom $\groundatom$ is treated as having a value of 0 if and only if
$\observations(\groundatom) = 0$.
Otherwise, it is treated as having a value of 1.
For example, imagine that we want to restrict the summation in the following arithmetic rule
to only constants that satisfy a property \texttt{Property/1}.
\[
\texttt{Link(X, +Y) <= 1 .}
\]
Then, we can add the following filter clause.
\[
\texttt{\{Y: Property(Y)\}}
\]
Then, the hard linear constraints templated by the arithmetic rule will only sum over constants
substituted for \texttt{Y} such that \texttt{Property(Y)} is non-zero.

In arithmetic rules, atoms can also be modified with coefficients.
These coefficients can be hard-coded.
As a simple example, in the rule
\[
\texttt{Susceptible(X) >= 0.5 Biomarker1(X) + 0.5 Biomarker2(X) .}
\]
the property \texttt{Susceptible/1}, which represents the degree to which
a patient is susceptible to a particular disease, must be at least the average value
of two biomarkers.

PSL also supports two forms of coefficient-defining syntax.
The first form of coefficient syntax is a cardinality function that counts the number of terms substituted
for a sum variable.
Cardinality functions enable rules that depend on the number of substitutions in order to be scaled
correctly, such as when averaging.
Cardinality is denoted by enclosing a sum variable, without the \texttt{+}, in pipes.
For example, the rule
\[
\texttt{1 / |Y| Friends(X, +Y) = Friendliness(X) .}
\]
defines the \texttt{Friendliness/1} property of a person \texttt{X} in a social network as the average
strength of their outgoing friendship links.
In cases in which \texttt{Friends/2} is not symmetric, we can extend this rule to sum over both
outgoing and incoming links as follows.
\begin{align*}
&\texttt{1 / |Y1| |Y2| Friends(X, +Y1) + 1 / |Y1| |Y2| Friends(+Y2, X)} \\
&\texttt{= Friendliness(X) .}
\end{align*}

The second form of coefficient syntax is built-in coefficient functions.
The exact set of supported functions is implementation specific, but standard functions like
maximum and minimum should be included.
Coefficient functions are prepended with \texttt{@} and use square brackets instead of parentheses to
distinguish them from predicates.
Coefficient functions can take either scalars or cardinality functions as arguments.
For example, the following rule for matching two sets of constants requires that the sum of the \texttt{Matched/2}
atoms be the minimum of the sizes of the two sets.
\[
\texttt{Matched(+X, +Y) = @Min[|X|, |Y|] .}
\]
Note that PSL's coefficient syntax can also be used to define constants, as in this example.

So far we have focused on using arithmetic rules to define templates for linear constraints,
but they can also be used to define hinge-loss potentials.
For example, the following arithmetic rule prefers that the degree to which a person \texttt{X} is extroverted (represented with \texttt{Extroverted/1}) does not exceed the average extroversion of their friends:
\begin{align*}
&\texttt{2 : Extroverted(X) <= 1 / |Y| Extroverted(+Y) \pslcaret2} \\
&\texttt{\{Y: Friends(X, Y) || Friends(Y, X)\}}
\end{align*}
This rule is a template for weighted hinge-loss potentials of the form
\begin{equation}
2 \left( \max{\left\{\egovar - \frac{1}{|\friendset|} \sum_{\iy \in \friendset} \y_i, 0 \right\}} \right)^2 \; ,
\end{equation}
where $\egovar$ is the variable corresponding to a grounding of the atom
\texttt{Extroverted(X)} and $\friendset$ is the set of the indices of the variables corresponding to
\texttt{Extroverted(Y)} atoms of the friends \texttt{Y} that satisfy the rule's filter clause.
Note that the weight of 2 is distinct from the coefficients in the linear constraint
$\linfun(\yy, \xx) \leq 0$ defining the hinge-loss potential.
If the arithmetic rule were an equality instead of an inequality, each grounding would be two hinge-loss
potentials, one using $\linfun(\yy, \xx) \leq 0$ and one using $-\linfun(\yy, \xx) \leq 0$.
In this way, arithmetic rules can define general hinge-loss potentials.

For completeness, we state the full, formal definition of an arithmetic rule and define its grounding procedure.
\begin{definition}
An {\bf arithmetic rule} is an inequality or equality relating two linear combinations of summation atoms.
Each sum variable in an arithmetic rule can be used once.
An arithmetic rule can be annotated with filter clauses for a subset of its sum variables that restrict its groundings.
Arithmetic rules are either weighted or unweighted.
If an arithmetic rule is weighted, it is annotated with a nonnegative weight and optionally a power of two.
\end{definition}
An arithmetic rule is grounded out by performing all distinct substitutions from variables to constants such that the resulting ground atoms are in the base $\base$.
In addition, summation atoms are replaced by the appropriate summations over ground atoms
(possibly restricted by corresponding filter clauses) and the coefficient is distributed across the summands.
This leads to a set of ground rules for each arithmetic rule given a set of inputs.
If the arithmetic rule is an unweighted inequality, each ground rule can be algebraically
manipulated to be of the form $\constr(\yy, \xx) \leq 0$.
Then $\constr(\yy, \xx)$ is added to the set of constraint functions and its index is added to $\Inequality$.
If instead the arithmetic rule is an unweighted equality, each ground rule is manipulated to
$\constr(\yy, \xx) = 0$, $\constr(\yy,\xx)$ is added to the set of constraint functions, and its index is
added to $\Equality$.
If the arithmetic rule is a weighted inequality with weight $\param$, each ground rule is
manipulated to $\linfun(\yy, \xx) \leq 0$ and included as a potential of the form
\begin{equation}
\pot(\yy, \xx) = \max{\left\{\linfun(\yy, \xx), 0 \right\}}
\end{equation}
with a weight of $\param$.
If the arithmetic rule is a weighted equality with weight $\param$, each ground rule
is again manipulated to $\linfun(\yy, \xx) \leq 0$ and two potentials are included,
\begin{equation}
\pot_1(\yy, \xx) = \max{\left\{\linfun(\yy, \xx), 0 \right\}}, ~~
\pot_2(\yy, \xx) = \max{\left\{-\linfun(\yy, \xx), 0 \right\}}~,
\end{equation}
each with a weight of $\param$.
In either case, if the weighted arithmetic rule is annotated with \texttt{\textasciicircum2},
then the induced potentials are squared.

\subsection{Expressivity}
\label{sec:expressivity}

An important question is the expressivity of PSL, which uses disjunctive clauses with positive weights
for its logical rules.
Other logic-based languages support different types of clauses, such as Markov logic networks
\citep{richardson:ml06}, which support clauses with conjunctions and clauses with negative weights.
As we discuss in this section, PSL's logical rules capture a general class of structural dependencies, capable of
modeling arbitrary probabilistic relationships among Boolean variables, such as those defined by Markov logic networks.
The advantage of PSL is that it defines HL-MRFs, which are much more scalable than discrete MRFs and
often just as accurate, as we show in Section~\ref{sec:experiments}.

The expressivity of PSL is tied to the expressivity of the MAX SAT problem, since they both use the same
class of weighted clauses.
There are two conditions on the clauses: (1) they have nonnegative weights, and (2) they are disjunctive.
We first consider the nonnegativity requirement and show that can actually be viewed as a restriction on the
structure of a clause.
To illustrate, consider a weighted disjunctive clause of the form
\begin{equation}
\label{eq:convert_neg}
-\param~~:~~\left( \bigvee_{\imsvar \in \msindicatorset_\imsclause^+} \msvar_\imsvar \right)
\bigvee
\left( \bigvee_{\imsvar \in \msindicatorset_\imsclause^-} \neg \msvar_\imsvar \right)~.
\end{equation}
If this clause were part of a generalized MAX SAT problem, in which there were no restrictions on weight sign or
clause structure, but the goal were still to maximize the sum of the weights of the satisfied clauses,
then this clause could be replaced with an equivalent one without changing the optimizer:
\begin{equation}
\label{eq:convert_pos}
\param~~:~~\left( \bigwedge_{\imsvar \in \msindicatorset_\imsclause^+} \neg \msvar_\imsvar \right)
\bigwedge
\left( \bigwedge_{\imsvar \in \msindicatorset_\imsclause^-} \msvar_\imsvar \right)~.
\end{equation}
Note that the clause has been changed in three ways: (1) the sign of the weight has been changed,
(2) the disjunctions have been replaced with conjunctions, and (3) the literals have all been negated.
Due to this equivalence, the restriction on the sign of the weights is subsumed by the restriction
on the structure of the clauses.
In other words, any set of clauses can be converted to a set with nonnegative weights that has the same
optimizer, but it might require including conjunctions in the clauses.
It is also easy to verify that if Equation~(\ref{eq:convert_neg}) is used to define a potential
in a discrete MRF, replacing it with a potential defined by~(\ref{eq:convert_pos}) leaves the distribution
unchanged, due to the normalizing partition function.

We now consider the requirement that clauses be disjunctive and illustrate how conjunctive clauses can be
replaced by an equivalent set of disjunctive clauses.
The idea is to construct a set of disjunctive clauses such that all assignments to the variables are mapped
to the same score, up to a constant.
A simple example is replacing a conjunction
\begin{equation}
\label{eq:convert_conjunct}
\param~~:~~\msvar_1 \wedge \msvar_2
\end{equation}
with disjunctions
\begin{align}
&\param~~:~~\msvar_1 \vee \msvar_2 \\
&\param~~:~~\neg \msvar_1 \vee \msvar_2 \\
&\param~~:~~\msvar_1 \vee \neg \msvar_2~.
\end{align}
Observe that the total score for all assignments to the variables remains the same, up to a constant.

This example generalizes to a procedure for encoding any Boolean MRF into
a set of disjunctive clauses with nonnegative weights.
\citet{park:aaai02} showed that the MAP problem for any discrete Bayesian network can be
represented as an instance of MAX SAT.
For distributions of bounded factor size, the MAX SAT problem has size polynomial in the number of
variables and factors of the distribution.
We describe how any Boolean MRF can be represented with disjunctive clauses and nonnegative weights.
Given a Boolean MRF with arbitrary potentials defined by mappings from joint states of subsets of
the variables to scores, a new MRF is created as follows.
For each potential in the original MRF, a new set of potentials defined by disjunctive clauses is created.
A conjunctive clause is created corresponding to each entry in the potential's mapping with a weight
equal to the score assigned by the weighted potential in the original MRF.
Then, these clauses are converted to equivalent disjunctive clauses as in the example of
Equations~(\ref{eq:convert_neg}) and~(\ref{eq:convert_pos}) by also flipping the sign of their weights 
and negating the literals.
Once this is done for all entries of all potentials, what remains is an MRF defined by disjunctive clauses,
some of which might have negative weights.
We make all weights positive by adding a sufficiently large constant to all weights of all clauses,
which leaves the distribution unchanged due to the normalizing partition function.

It is important to note two caveats when converting arbitrary Boolean MRFs to MRFs defined using only
disjunctive clauses with nonnegative weights.
First, the number of clauses required to represent a potential in the original MRF is exponential
in the degree of the potential.
In practice, this is rarely a significant limitation, since MRFs often contain low-degree potentials.
The other important point is that the step of adding a constant to all the weights increases the
total score of the MAP state.
Since the bound of \citet{goemans:discmath94} is relative to this score, the bound is loosened for
the original problem the larger the constant added to the weights is.
This is to be expected, since even approximating MAP is NP-hard in general \citep{abdelbar:ai98}.

We have described how general structural dependencies can be modeled with the logical rules of PSL.
It is possible to represent arbitrary logical relationships with them.
The process for converting general rules to PSL's logical rules can be done automatically and made
transparent to the user.
We have elected in this section to define PSL's logical rules without making this conversion automatic to make clear the
underlying formalism.

\subsection{Modeling Patterns}
\label{sec:patterns}

PSL is a flexible language, and there are some patterns of usage that come up in many applications.
We illustrate some of them in this subsection with a number of examples.

\subsubsection{Domain and Range Rules}
\label{sec:domain_and_range}

In many problems, the number of relations that can be predicted among some constants is known.
For binary predicates, this background knowledge can be viewed as constraints on the domain (first
argument) or range (second argument) of the predicate.
For example, it might be background knowledge that each entity, such as a document, has
exactly one label.
An arithmetic rule to express this follows.
\[
\texttt{Label(Document, +LabelName) = 1 .}
\]
The predicate \texttt{Label} is said to be {\em functional}.

Alternatively, sometimes it is the first argument that should be summed over.
For example, imagine the task of predicting relationships among students and professors.
Perhaps it is known that each student has exactly one advisor.
This constraint can be written as follows.
\[
\texttt{Advisor(+Professor, Student) = 1 .}
\]
The predicate \texttt{Advisor} is said to be {\em inverse functional}.

Finally, imagine a scenario in which two social networks are being aligned.
The goal is to predict whether each pair of people, one from each network, is the same person,
which is represented with atoms of the \texttt{Same} predicate.
Each person aligns with at most one person in the other network, but might not align with anyone.
This can be expressed with the following two arithmetic rules.
\begin{align*}
& \texttt{Same(Person1, +Person2) <= 1 .} \\
& \texttt{Same(+Person1, Person2) <= 1 .}
\end{align*}
The predicate \texttt{Same} is said to be both {\em partial functional} and {\em partial inverse functional}.

Many variations on these examples are possible.
For example, they can be generalized to predicates with more than two arguments.
Additional arguments can either be fixed or summed over in each rule.
As another example, domain and range rules can incorporate multiple predicates, so that
an entity can participate in a fixed number of relations counted among multiple predicates.

\subsubsection{Similarity}

\label{sec:similarity}

Many problems require explicitly reasoning about similarity, rather than simply whether entities are the same or different.
For example, reasoning with similarity has been explored using kernel methods, such as kFoil \citep{landwehr:ml10} that bases
similarity computation on the relational structure of the data.
The continuous variables of HL-MRFs make modeling similarity straightforward, and PSL's support for functionally defined
predicates makes it even easier.
For example, in an entity resolution task, the degree to which two entities are believed to be the same might
depend on how similar their names are.
A rule expressing this dependency is
\[
\texttt{1.0 : Name(P1, N1) \psland\  Name(P2, N2) \psland\ Similar(N1, N2) \pslimpliesright\ Same(P1, P2)}
\]
This rule uses the \texttt{Similar} predicate to measure similarity.
Since it is a functionally defined predicate, it can be implemented as one of many different, possibly domain specialized,
string similarity functions.
Any similarity function that can output values in the range $[0,1]$ can be used.

\subsubsection{Priors}

If no potentials are defined over a particular atom, then it is equally probable that it has any value
between zero and one.
Often, however, it should be more probable that an atom has a value of zero,
unless there is evidence that it has a nonzero value.
Since atoms typically represent the existence of some entity, attribute, or relation, this bias promotes sparsity
among the things inferred to exist.
Further, if there is a potential that prefers that an atom should have a value that is at least some numeric constant,
such as when reasoning with similarities as discussed in Section~\ref{sec:similarity}, it often should also be more probable that
an atom is no higher in value than is necessary to satisfy that potential.
To accomplish both these goals, simple priors can be used to state that atoms should have low values in the
absence of evidence to overrules those priors.
A prior in PSL can be a rule consisting of just a negative literal with a small weight.
For example, in a link prediction task, imagine that this preference should apply to atoms of the \texttt{Link} predicate.
A prior is then
\[
\texttt{0.1 : !Link(A, B)}
\]
which acts as a regularizer on \texttt{Link} atoms.

\subsubsection{Blocks and Canopies}

In many tasks, the number of unknowns can quickly grow large, even for modest amounts of data.
For example, in a link prediction task the goal is to predict relations among entities.
The number of possible links grows quadratically with the number of entities (for binary relations).
If handled naively, this growth could make scaling to large data sets difficult, but this problem is often handled by
constructing {\em blocks} \citep[e.g.,][]{newcombe:cacm62} or {\em canopies} \citep{mccallum:kdd00} over the entities,
so that a limited subset of all possible links are actually considered.
Blocking partitions the entities so that only links among entities in the same partition element, i.e., block, are considered.
Alternatively, for a finer grained pruning, a canopy is defined for each entity, which is the set of other entities to which it could possibly link.
Blocks and canopies can be computed using specialized, domain-specific functions, and
PSL can incorporate them by including them as atoms in the bodies of rules.
Since blocks can be seen as a special case of canopies, we let the atom \texttt{InCanopy(A,~B)} be 1 if \texttt{B}
is in the canopy or block of \texttt{A}, and 0 if it is not.
Including \texttt{InCanopy(A,~B)} atoms as additional conditions in the bodies of logical rules will ensure that
the dependencies only exist between the desired entities.

\subsubsection{Aggregates}
\label{sec:aggregates}

Another powerful feature of PSL is its ability to easily define \emph{aggregates},
which are rules that define random variables to be deterministic functions of sets of other random variables.
The advantage of aggregates is that they can be used to define dependencies that do not scale
in magnitude with the number of groundings in the data.
For example, consider a model for predicting interests in a social network.
A fragment of a PSL program for this task follows.
\begin{align*}
&\texttt{1.0 : Interest(P1, I) \psland\  Friends(P1, P2) \pslimpliesright\ Interest(P2, I)} \\
&\texttt{1.0 : Age(P, "20-29") \psland\ Lives(P, "California") \pslimpliesright\ Interest(P, "Surfing")}
\end{align*}
These two rules express the belief that interests are correlated along friendship links in the social network,
and also that certain demographic information is predictive of specific interests.
The question any domain expert or learning algorithm faces is how strongly each rule should be
weighted relative to each other.
The challenge of answering this question when using templates is that the number of groundings of
the first rule varies from person to person based on the number of friends, while the groundings
of the second remain constant (one per person).
This inconsistent scaling of the two types of dependencies makes it difficult to find weights that
accurately reflect the relative influence each type of dependency should have across people with
different numbers of friends.

Using an aggregate can solve this problem of inconsistent scaling.
Instead of using a separate ground rule to relate the interest of each friend, we can define a rule
that is only grounded once for each person, relating an average interest across all friends to each
person's own interests.
A PSL fragment for this approach is
\begin{align*}
&\texttt{1.0 : AverageFriendInterest(P, I) \pslimpliesright\ Interest(P, I)} \\
&\texttt{AverageFriendInterest(P, I) = 1 / |F| Interest(+F, I) .} \\
&\texttt{\{F: Friends(P, F)\}} \\
&\phantom{a} \\
&\texttt{/* Demographic dependencies are also included. */}
\end{align*}
where the predicate \texttt{AverageFriendInterest/2} is an aggregate that is constrained to be
the average amount of interest each friend of a person \texttt{P} has in an interest \texttt{I}.
The weight of the logical rule can now be scaled more appropriately relative to other types of features because
there is only one grounding per person.

For a more complex example, consider the problem of determining whether two references in the data
refer to the same underlying person.
One useful feature to use is whether they have similar sets of friends in the social network.
Again, a rule could be defined that is grounded out for each friendship pair, but this would suffer from
the same scaling issues as the previous example.
Instead, we can use an aggregate to directly express how similar the two references' sets of friends are.
A function that measures the similarity of two sets $A$ and $B$ is \emph{Jaccard similarity}:
\[
J(A, B) = \frac{|A \cap B|}{|A \cup  B|}~.
\]
Jaccard similarity is a nonlinear function, meaning that it cannot be used directly without breaking the
log-concavity of HL-MRFs, but we can approximate it with a linear function.
We define \texttt{SameFriends/2} as an aggregate that approximates Jaccard similarity
(where \texttt{SamePerson/2} is functional and inverse functional).
\begin{align*}
&\texttt{SameFriends(A, B) = 1 / @Max[|FA|, |FB|] SamePerson(+FA, +FB) . } \\
&\texttt{\{FA : Friends(A, FA)\}} \\
&\texttt{\{FB : Friends(B, FB)\}} \\
&\texttt{SamePerson(+P1, P2) = 1 . } \\
&\texttt{SamePerson(P1, +P2) = 1 . }
\end{align*}
The aggregate \texttt{SameFriends/2} uses the sum of the \texttt{SamePerson/2} atoms
as the intersection of the two sets, and the maximum of the sizes of the two sets of friends
as a lower bound on the size of their union.


\section{MAP Inference}
\label{sec:map}

Having defined HL-MRFs and a language for creating them, PSL, we turn to algorithms for inference and learning.
The first task we consider is maximum a posteriori (MAP) inference, the problem
of finding a most probable assignment to the free variables $\yy$ given observations $\xx$.
In HL-MRFs, the normalizing function $\partition(\pparam, \xx)$ is constant over $\yy$ and the
exponential is maximized by minimizing its negated argument, so the MAP problem is
\begin{equation}
\label{eq:map}
\begin{aligned}
\argmax_\yy \prob(\yy | \xx) \; \; \; \equiv \; \; \; &\argmin_{\yy|\yy,\xx \in \FeasibleDom} \energy(\yy,\xx) & \\
\equiv \; \; \; &\argmin_{\yy \in [0,1]^\ny} \pparam^\top \ppot(\yy,\xx) & \\
&\phantom{a} & \\ 
&\text{such that} &\constr_\iconstr(\yy,\xx) = 0, \; \; \; \forall \iconstr \in \Equality~\phantom{.}\\
& &\constr_\iconstr(\yy,\xx) \leq 0, \; \; \; \forall \iconstr \in \Inequality~.
\end{aligned}
\end{equation}
MAP is a fundamental problem because (1) it is the method we will use to make predictions, and
(2) weight learning often requires performing MAP inference many times with different weights
(as we discuss in Section~\ref{sec:learning}).
Here, HL-MRFs have a distinct advantage over general discrete models, since minimizing $\energy$ is a
convex optimization rather than a combinatorial one.
There are many off-the-shelf solutions for convex optimization, the most popular
of which are interior-point methods, which have worst-case polynomial time complexity in the number
of variables, potentials, and constraints \citep{nesterov:book94}.
Although in practice they perform better than their worst-case bounds \citep{wright:ams05},
they do not scale well to large structured prediction problems \citep{yanover:jmlr06}.
We therefore introduce a new algorithm for exact MAP inference designed to scale to large HL-MRFs
by leveraging the sparse connectivity structure of the potentials and hard constraints that are typical
of models for real-world tasks.

\subsection{Consensus Optimization Formulation}
\label{sec:conopt}

Our algorithm uses {\em consensus optimization}, a technique that divides an optimization problem
into independent subproblems and then iterates to reach a consensus on the optimum \citep{boyd:book11}.
Given a HL-MRF $\prob(\yy|\xx)$, we first construct an equivalent MAP problem in which each potential
and hard constraint is a function of different variables.
The variables are then constrained to make the new and original MAP problems equivalent.
We let $\localyyi{\ipot}$ be a local copy of the variables in $\yy$ that are used in the potential function
$\pot_\ipot$, $\ipot=1,\dots, \npot$ and $\localyyi{\iconstr+\npot}$ be a copy of those used
in the constraint function $\constr_\iconstr$, $\iconstr=1,\dots, \nconstr$.
We refer to the concatenation of all of these vectors as $\localyy$.
We also introduce a characteristic function $\indicator_\iconstr$ for each constraint function where
$\indicator_\iconstr \left[ \constr_\iconstr( \localyyi{\iconstr+\npot}, \xx ) \right] = 0$ if the
constraint is satisfied and infinity if it is not.
Likewise, let $\indicator_{[0,1]}$ be a characteristic function that is 0 if the input is in the interval $[0,1]$
and infinity if it is not.
We drop the constraints on the domain of $\yy$, letting them range in principle over $\mathbb{R}^\ny$
and instead use these characteristic functions to enforce the domain constraints.
This formulation will make computation easier when the problem is later decomposed.
Finally, let $\correspondingyyi{\isubproblem}$ be the variables in $\yy$ that correspond to
$\localyyi{\isubproblem}$, $\isubproblem = 1,\dots, \npot+\nconstr$.
Operators between $\localyyi{\isubproblem}$ and $\correspondingyyi{\isubproblem}$ are defined
element-wise, pairing the corresponding copied variables.
Consensus optimization solves the reformulated MAP problem
\begin{equation}
\label{eq:modmap}
\begin{aligned}
&\argmin_{ (\localyy, \yy) }
&\sum_{\ipot=1}^\npot{\param_\ipot \pot_\ipot \left( \localyyi{\ipot}, \xx \right) }
+ \sum_{\iconstr=1}^\nconstr{ \indicator_\iconstr \left[ \constr_\iconstr
\left( \localyyi{\iconstr+\npot}, \xx \right) \right] } + \sum_{\iy=1}^\ny \indicator_{[0,1]} \left[ \y_\iy \right] \\
& \phantom{a} & \\
&\text{such that}  & \localyyi{\isubproblem} = \correspondingyyi{\isubproblem} \; \; \;
\forall \isubproblem = 1,\dots,\npot + \nconstr~.
\end{aligned}
\end{equation}
Inspection shows that problems~(\ref{eq:map}) and~(\ref{eq:modmap}) are equivalent.

This reformulation enables us to relax the equality constraints
$\localyyi{\isubproblem} = \correspondingyyi{\isubproblem}$ in order to divide problem~(\ref{eq:modmap})
into independent subproblems that are easier to solve,
using the alternating direction method of multipliers (ADMM) \citep{glowinski:france75, gabay:cma76, boyd:book11}.
The first step is to form the \emph{augmented Lagrangian} function for the problem.
Let $\llagrangemult = (\llagrangemult_1,\dots, \llagrangemult_{\npot + \nconstr})$ be a
concatenation of vectors of Lagrange multipliers.
Then the augmented Lagrangian is
\begin{multline}
\lagrangian(\localyy, \llagrangemult, \yy)
= \sum_{\ipot=1}^\npot{\param_\ipot \pot_\ipot \left( \localyyi{\ipot}, \xx \right) }
+ \sum_{\iconstr=1}^\nconstr{ \indicator_\iconstr \left[ \constr_\iconstr \left( \localyyi{\iconstr+\npot}, \xx \right) \right] }
+ \sum_{\iy=1}^\ny \indicator_{[0,1]} \left[ \y_\iy \right] \\
+ \sum_{\isubproblem = 1}^{\npot + \nconstr} \llagrangemult_\isubproblem^\top \left(\localyyi{\isubproblem} - \correspondingyyi{\isubproblem}\right)
+ \frac{\admmstep}{2} \sum_{\isubproblem = 1}^{\npot + \nconstr} \left\| \localyyi{\isubproblem} - \correspondingyyi{\isubproblem} \right\|_2^2
\end{multline}
using a step-size parameter $\admmstep > 0$.
ADMM finds a saddle point of  $\lagrangian(\localyy, \llagrangemult, \yy)$ by
updating the three blocks of variables at each iteration $\timestep$:
\begin{align}
\llagrangemult_\isubproblem^\timestep &\leftarrow \llagrangemult_\isubproblem^{\timestep-1}
+ \admmstep \left(\localyyi[\timestep-1]{\isubproblem} - \correspondingyyi[\timestep-1]{\isubproblem} \right)
&  \forall \isubproblem = 1,\dots,\npot+\nconstr  \label{eq:admmyupdate} \\[0.5em]
\localyy[\timestep] &\leftarrow \argmin_{\localyy}~\lagrangian\left(\localyy, \llagrangemult^\timestep, \yy^{\timestep-1}\right) & \label{eq:admmxupdate} \\[0.5em]
\yy^\timestep &\leftarrow \argmin_{\yy}~\lagrangian\left(\localyy[\timestep], \llagrangemult^\timestep, \yy\right) \label{eq:admmzupdate} &
\end{align}

The ADMM updates ensure that $\yy$ converges to the global optimum $\yy^\star$, the MAP
state of $\prob(\yy | \xx)$, assuming that there exists a feasible assignment to $\yy$.
We check convergence using the criteria suggested by \citet{boyd:book11}, measuring the primal
and dual residuals at the end of iteration $\timestep$, defined as
\begin{equation}
\| \bar{{\boldsymbol r}}^\timestep \|_2 \triangleq \left( \sum_{\isubproblem = 1}^{\npot + \nconstr} \|\localyyi[\timestep]{\isubproblem} - \correspondingyyi[\timestep]{\isubproblem} \|^2_2 \right)^{\frac{1}{2}}~~~~~~~~
\| \bar{{\boldsymbol s}}^\timestep \|_2 \triangleq \admmstep \left( \sum_{\iy = 1}^\ny \numcopies_\iy (\y_\iy^\timestep - \y_\iy^{\timestep-1})^2 \right)^{\frac{1}{2}}
\end{equation}
where $\numcopies_\iy$ is the number of copies made of the variable $\y_\iy$, i.e., the number of
different potentials and constraints in which the variable participates.
The updates are terminated when both of the following conditions are satisfied
\begin{align}
\| \bar{{\boldsymbol r}}^\timestep \|_2 & \leq \epsilon^\text{abs} \sqrt{\sum_{\iy = 1}^\ny \numcopies_\iy}
+ \epsilon^\text{rel} \max \left\{
\left( \sum_{\isubproblem = 1}^{\npot + \nconstr} \| \localyyi[\timestep]{\isubproblem} \|_2^2 \right)^\frac{1}{2},
\left( \sum_{\iy = 1}^\ny \numcopies_\iy (\y_\iy^\timestep)^2 \right)^\frac{1}{2}
\right\} \\
\| \bar{{\boldsymbol s}}^\timestep \|_2 & \leq\epsilon^\text{abs} \sqrt{\sum_{\iy = 1}^\ny \numcopies_\iy}
+ \epsilon^\text{rel} \left( \sum_{\isubproblem = 1}^{\npot + \nconstr} \|\llagrangemult_\isubproblem^\timestep \|_2^2 \right)^\frac{1}{2}
\end{align}
using convergence parameters $\epsilon^\text{abs}$ and $\epsilon^\text{rel}$.

\subsection{Block Updates}	
\label{sec:block_updates}

We now describe how to implement the ADMM block updates~(\ref{eq:admmyupdate}),~(\ref{eq:admmxupdate}),
and~(\ref{eq:admmzupdate}).
Updating the Lagrange multipliers $\llagrangemult$ is a simple step in the gradient direction~(\ref{eq:admmyupdate}).
Updating the local copies $\localyy$~(\ref{eq:admmxupdate}) decomposes over each potential and constraint in the HL-MRF.
For the variables $\localyyi{\ipot}$ for each potential $\pot_\ipot$, this requires independently
optimizing the weighted potential plus a squared norm:
\begin{equation}
\label{eq:potsubproblem}
\argmin_{\localyyi{\ipot}}~\param_\ipot \left( \max\left\{ \linfun_\ipot(\localyyi{\ipot}, \xx), 0 \right\} \right)^{\p_\ipot}
     + \frac{\admmstep}{2} \left\| \localyyi{\ipot} - \correspondingyyi{\ipot} + \frac{1}{\rho} \llagrangemult_\ipot \right\|_2^2~.
\end{equation}
Although this optimization problem is convex, the presence of the hinge function complicates it.
It could be solved in principle with an iterative method, such as an interior-point method, but such methods would become
very expensive over many ADMM updates.
Fortunately, we can reduce the problem to checking several cases and find solutions much more quickly.

There are three cases for $\localyyi[\star]{\ipot}$, the optimizer of problem~(\ref{eq:potsubproblem}), which correspond to the
three regions in which the solution could lie:
(1) the region $\linfun(\localyyi{\ipot}, \xx) < 0$,
(2) the region $\linfun(\localyyi{\ipot}, \xx) > 0$,
and (3) the region $\linfun(\localyyi{\ipot}, \xx) = 0$.
We check each case by replacing the potential with its value on the corresponding region, optimizing, and checking
if the optimizer is in the correct region.
We check the first case by replacing the potential $\pot_\ipot$ with zero.
Then, the optimizer of the modified problem is $\correspondingyyi{\ipot} - \llagrangemult_\ipot / \rho$.
If $\linfun_\ipot(\correspondingyyi{\ipot} - \llagrangemult_\ipot / \rho, \xx) \leq 0$, then 
$\localyyi[\star]{\ipot} = \correspondingyyi{\ipot} - \llagrangemult_\ipot / \rho$,
because it optimizes both the potential and the squared norm independently.
If instead $\linfun_\ipot(\correspondingyyi{\ipot} - \llagrangemult_\ipot / \rho, \xx) > 0$,
then we can conclude that $\linfun_\ipot(\localyyi[\star]{\ipot}, \xx) \geq 0$, leading to one of the next two cases.

In the second case, we replace the maximum term with the inner linear function.
Then the optimizer of the modified problem is found by taking the gradient of the objective with respect to $\localyyi{\ipot}$,
setting the gradient equal to the zero vector, and solving for $\localyyi{\ipot}$.
In other words, the optimizer is the solution for $\localyyi{\ipot}$ to the equation
\begin{equation}
\label{eq:potsystem}
\nabla_{\localyyi{\ipot}}
\left[
\param_\ipot \left( \linfun_\ipot(\localyyi{\ipot}, \xx) \right)^{\p_\ipot}
     + \frac{\admmstep}{2} \left\| \localyyi{\ipot} - \correspondingyyi{\ipot} + \frac{1}{\rho} \llagrangemult_\ipot \right\|_2^2
\right] = {\boldsymbol 0}~.
\end{equation}
This condition defines a simple system of linear equations.
If $\p_\ipot = 1$, then the coefficient matrix is diagonal and trivial to solve.
If $\p_\ipot = 2$, then the coefficient matrix is symmetric and positive definite, and the system can be solved
via Cholesky decomposition.
(Since the potentials of an HL-MRF often have shared structures, perhaps templated by a PSL program, the
Cholesky decompositions can be cached and shared among potentials for improved performance.)
Let $\localyyi[\prime]{\ipot}$ be the optimizer of the modified problem, i.e., the solution to equation~(\ref{eq:potsystem}).
If $\linfun_\ipot(\localyyi[\prime]{\ipot}, \xx) \geq 0$, then $\localyyi[\star]{\ipot} = \localyyi[\prime]{\ipot}$ because
we know the solution lies in the region $\linfun_\ipot(\localyyi{\ipot}, \xx) \geq 0$ and the objective of
problem~(\ref{eq:potsubproblem}) and the modified objective are equal on that region.
In fact, if $\p_\ipot = 2$, then $\linfun_\ipot(\localyyi[\prime]{\ipot}, \xx) \geq 0$ whenever
$\linfun_\ipot(\correspondingyyi{\ipot} - \llagrangemult_\ipot / \rho, \xx) \geq 0$, because
 the modified term is symmetric about the line $\linfun_\ipot(\localyyi{\ipot}, \xx) = 0$.
We therefore will only reach the following third case when $\p_\ipot = 1$.
If $\linfun_\ipot(\correspondingyyi{\ipot} - \llagrangemult_\ipot / \rho, \xx) > 0$ and
$\linfun_\ipot(\localyyi[\prime]{\ipot}, \xx) < 0$, then we can conclude that $\localyyi[\star]{\ipot}$
is the projection of $\correspondingyyi{\ipot} - \llagrangemult_{\ipot} / \rho$ onto the hyperplane
$\constr_\iconstr(\localyyi{\ipot}, \xx) = 0$.
This constraint must be active because it is violated by the optimizers of both modified objectives \citep[Lemma 17]{martins:jmlr15}.
Since the potential has a value of zero whenever the constraint is active, solving problem~(\ref{eq:potsubproblem}) reduces
to the projection operation.

For the local copies $\localyyi{\iconstr + \npot}$ for each constraint $\constr_\iconstr$,
the subproblem is easier:
\begin{equation}
\argmin_{\localyyi{\iconstr+\npot}}~\indicator_\iconstr \left[ \constr_\iconstr( \localyyi{\iconstr+\npot}, \xx ) \right] 
     + \frac{\admmstep}{2} \left\| \localyyi{\iconstr + \npot} - \correspondingyyi{\iconstr+\npot} + \frac{1}{\rho} \llagrangemult_{\iconstr+\npot} \right\|_2^2~.
\end{equation}
Whether $\constr_\iconstr$ is an equality or inequality constraint, the solution is the projection
of $\correspondingyyi{\iconstr+\npot} - \llagrangemult_{\iconstr+\npot} / \rho$ to the
feasible set defined by the constraint.
If $\constr_\iconstr$ is an equality constraint, i.e., $\iconstr \in \Equality$, then the optimizer
$\localyyi[\star]{\iconstr+\npot}$ is the projection of $\correspondingyyi{\iconstr+\npot} - \llagrangemult_{\iconstr+\npot} / \rho$
onto $\constr_\iconstr(\localyyi{\iconstr+\npot}, \xx) = 0$.
If, on the other hand, $\constr_\iconstr$ is an inequality constraint, i.e., $\iconstr \in \Inequality$, then
there are two cases.
First, if $\constr_\iconstr(\correspondingyyi{\iconstr+\npot} - \llagrangemult_{\iconstr+\npot} / \rho, \xx) \leq 0$,
then the solution is simply $\correspondingyyi{\iconstr+\npot} - \llagrangemult_{\iconstr+\npot} / \rho$.
Otherwise, it is again the projection onto $\constr_\iconstr(\localyyi{\iconstr+\npot}, \xx) = 0$.

To update the variables $\yy$~(\ref{eq:admmzupdate}), we solve the optimization
\begin{equation}
\argmin_{\yy}~\sum_{\iy=1}^\ny \indicator_{[0,1]} \left[ \y_\iy \right] + \frac{\admmstep}{2} \sum_{\isubproblem = 1}^{\npot + \nconstr} \left\| \localyyi{\isubproblem} - \correspondingyyi{\isubproblem} + \frac{1}{\rho} \llagrangemult_{\isubproblem} \right\|_2^2.
\end{equation}
The optimizer is the state in which $\y_\iy$ is set to the average of its corresponding local copies added with
their corresponding Lagrange multipliers divided by the step size $\admmstep$,
and then clipped to the $[0,1]$ interval.
More formally, let $\mathtt{copies}(\y_\iy)$ be the set of local copies $\y_c$ of $\y_\iy$, each with a corresponding Lagrange multiplier
$\lagrangemult_c$.
Then, we update each $\y_\iy$ using
\begin{equation}
\label{eq:consensusupdate}
\y_\iy \leftarrow \frac{1}{|\mathtt{copies}(\y_\iy)|}   \sum_{y_c \in \mathtt{copies}(\y_\iy)}{\left( y_c + \frac{\alpha_c}{\rho}\right)}
\end{equation}
and clip the result to $[0,1]$.
Specifically, if, after update~(\ref{eq:consensusupdate}), $\y_\iy > 1$, then we set $\y_\iy$ to 1 and likewise set it to 0
if $\y_\iy < 0$.

\begin{algorithm}
\caption{MAP Inference for HL-MRFs}
\label{alg:co}
\begin{algorithmic}
\STATE {\bfseries Input:} HL-MRF $\prob(\yy | \xx)$, $\admmstep > 0$
\vskip 0.4em

\STATE Initialize $\localyyi{\ipot}$ as local copies of variables $\correspondingyyi{\ipot}$ that are in $\pot_\ipot$, $j=1,\dots, \npot$

\STATE Initialize $\localyyi{\iconstr+\npot}$ as local copies of variables $\correspondingyyi{\iconstr+\npot}$ that are in $\constr_\iconstr$, $\iconstr=1,\dots, \nconstr$

\STATE Initialize Lagrange multipliers $\llagrangemult_\isubproblem$ corresponding to copies $\localyyi{\isubproblem}$, $\isubproblem=1,\dots, \npot+\nconstr$
\vskip 0.5em

\WHILE {not converged}
\vskip 0.5em

  \FOR {$\ipot = 1,\dots, \npot$}
  \STATE $\llagrangemult_\ipot \leftarrow \llagrangemult_\ipot + \admmstep(\localyyi{\ipot}-\correspondingyyi{\ipot})$
  \STATE $\localyyi{\ipot} \leftarrow \correspondingyyi{\ipot} - \frac{1}{\rho}\llagrangemult_\ipot$
  \IF {$\linfun_\ipot(\localyyi{\ipot}, \xx) > 0$}
    \STATE $\localyyi{\ipot} \leftarrow \argmin_{ \localyyi{\ipot}} 
    \param_\ipot \left( \linfun_\ipot(\localyyi{\ipot}, \xx) \right)^{\p_\ipot}
     + \frac{\admmstep}{2} \left\| \localyyi{\ipot} - \correspondingyyi{\ipot} + \frac{1}{\rho} \llagrangemult_\ipot \right\|_2^2 $
    \IF {$\linfun_\ipot(\localyyi{\ipot}, \xx) < 0$}
      \STATE $\localyyi{\ipot} \leftarrow \project_{\linfun_\ipot = 0} (\correspondingyyi{\ipot} - \frac{1}{\rho} \llagrangemult_{\ipot})$
    \ENDIF
  \ENDIF
  \ENDFOR
\vskip 0.5em

  \FOR {$\iconstr = 1,\dots, \nconstr$}
  \STATE $\llagrangemult_{\iconstr+\npot} \leftarrow \llagrangemult_{\iconstr+\npot} + \rho (\localyyi{\iconstr+\npot} - \correspondingyyi{\iconstr+\npot})$
  \STATE $\localyyi{\iconstr+\npot} \leftarrow \project_{\constr_\iconstr}(\correspondingyyi{\iconstr + \npot} - \frac{1}{\rho} \llagrangemult_{\iconstr+\npot})$
  \ENDFOR
\vskip 0.5em

  \FOR {$\iy = 1,\dots, \ny$}
  \STATE $\y_\iy \leftarrow \frac{1}{|\mathtt{copies}(\y_\iy)|}   \sum_{y_c \in \mathtt{copies}(\y_\iy)}{\left( y_c + \frac{\alpha_c}{\rho}\right)}$
\vskip 0.1em
  \STATE Clip $\y_\iy$ to [0,1]
  \ENDFOR
  \vskip 0.5em

\ENDWHILE
\end{algorithmic}
\end{algorithm}

Algorithm \ref{alg:co} shows the complete pseudocode for MAP inference.
The method starts by initializing local copies
of the variables that appear in each potential and constraint, along with a corresponding Lagrange
multiplier for each copy.
Then, until convergence, it iteratively performs the updates~(\ref{eq:admmyupdate}),~(\ref{eq:admmxupdate}),
and~(\ref{eq:admmzupdate}).
In the pseudocode, we have interleaved updates~(\ref{eq:admmyupdate}) and~(\ref{eq:admmxupdate}),
updating both the Lagrange multipliers $\llagrangemult_\isubproblem$ and the local copies $\localyyi{\isubproblem}$
together for each subproblem, because they are local operations that do not depend on other variables
once $\yy$ is updated in the previous iteration.
This independence reveals another advantage of our inference algorithm: it is very easy to parallelize.
The updates~(\ref{eq:admmyupdate}) and~(\ref{eq:admmxupdate}) can be performed in parallel, the results gathered,
update~(\ref{eq:admmzupdate}) performed, and the updated $\yy$ broadcast back to the subproblems.
Parallelization makes our MAP inference algorithm even faster and more scalable.

\subsection{Lazy MAP Inference}
\label{sec:lazymap}

\newcommand{\lazysubset}{\hat{\ppot}}

One interesting and useful property of HL-MRFs is that it is not always necessary to completely
materialize the distribution in order to find a MAP state.
Consider a subset $\lazysubset$ of the index set $\{1,\dots,\npot\}$ of the potentials $\ppot$.
Observe that if a feasible assignment to $\yy$ minimizes
\begin{equation}
\sum_{\ipot \in \lazysubset} \param_\ipot \pot_\ipot(\yy, \xx)
\end{equation}
and $\pot_\ipot(\yy, \xx) = 0, \forall \ipot \notin \lazysubset$, then that assignment must be a MAP
state because 0 is the global minimum for any potential.
Therefore, if we can identify a set of potentials that is small, such that all the other potentials
are 0 in a MAP state, then we can perform MAP inference in a reduced amount of time.
Of course, identifying this set is as hard as MAP inference itself, but we can iteratively grow the set
by starting with an initial set, performing inference over the current set, adding any potentials that have
nonzero values, and repeating.

Since the lazy inference procedure requires that the assignment be feasible, there are two ways to handle any constraints
in the HL-MRF.
One is to include all constraints in the inference problem from the beginning.
This strategy ensures feasibility, but the idea of lazy grounding can also be extended to constraints to improve performance further.
Just as we check if potentials are unsatisfied, i.e., nonzero, we can also check if constraints are unsatisfied, i.e., violated.
So the algorithm now iteratively grows the set of active potentials and active constraints, adding any that are unsatisfied
until the MAP state of the HL-MRF defined by the active potentials and constraints is also a feasible MAP state of the true HL-MRF.

The efficiency of lazy MAP inference can be improved heuristically by not adding all unsatisfied potentials and constraints,
but instead only adding those that are unsatisfied by some threshold.
This heuristic can decrease computational cost significantly, although the results are no longer guaranteed to be correct.
Bounding the resulting error when possible is an important direction for future work.

\subsection{Evaluation of MAP Inference}
\label{sec:map_experiments}

In this section we evaluate the empirical performance of our MAP inference algorithm.\footnote{Code is available at \url{https://github.com/stephenbach/bach-jmlr17-code}.} We compare its running times against those of
MOSEK,\footnote{http://www.mosek.com} a commercial convex optimization toolkit that uses interior-point methods (IPMs).
We confirm the results of \citet{yanover:jmlr06} that IPMs do not scale well to large structured-prediction problems,
and we show that our MAP inference algorithm scales much better.
In fact, we observe that our method scales linearly in practice with the number of potentials and constraints in the HL-MRF.

We evaluate scalability by generating social networks of varying sizes,
constructing HL-MRFs over them, and measuring the running time required to find a MAP state.
We compare our algorithm to MOSEK's IPM.
The social networks we generate are designed to be representative of common social-network analysis tasks.
We generate networks of users that are connected by different types of relationships, such as friendship and marriage,
and our goal is to predict the political preferences, e.g., liberal or conservative, of each user.
We also assume that we have local information about each user, representing features such as demographic information.

We generate the social networks using power-law distributions according to a procedure described
by \citet{broecheler:socialcom10}.
For a target number of users $N$, in-degrees and out-degrees $d$ for each edge type are sampled from the power-law distribution
$D(k) \equiv \alpha k^{-\gamma}$.
Incoming and outgoing edges of the same type are then matched randomly to create edges
until no more matches are possible.
The number of users is initially the target number plus the expected number of users with zero edges,
and then users without any edges are removed.
We use six edge types with various parameters to represent relationships
in social networks with different combinations of  abundance and exclusivity,
choosing $\gamma$ between 2 and 3,  and $\alpha$ between 0 and 1, as suggested by Broecheler et al.
We then annotate each vertex with a value in $[-1,1]$ uniformly at random to represent local features indicating one political preference or the other.

We generate social networks with between 22k and 66k vertices,
which induce HL-MRFs with between 130k and 397k total potentials and constraints.
In all the HL-MRFs, roughly 85\% of those totals are potentials.
For each social network, we create both a (log) piecewise-linear HL-MRF
($\p_\ipot=1,\forall \ipot = 1,\dots,\npot$ in Definition~\ref{def:energy})
and a piecewise-quadratic one ($\p_\ipot=2,\forall \ipot = 1,\dots,\npot$).
We weight local features with a parameter of $0.5$ and choose parameters in $[0, 1]$ for the relationship potentials representing a mix of more and less influential relationships.

We implement ADMM in Java and compare with the IPM in MOSEK (version 6) by encoding the entire MPE problem as a
linear program or a second-order cone program as appropriate and passing the encoded problem
via the Java native interface wrapper.
All experiments are performed on a single machine with a 4-core 3.4 GHz Intel Core i7-3770 processor with 32GB of RAM.
Each optimizer used a single thread, and all results are averaged over 3 runs.

\begin{figure}
\centering
\label{fig:map_experiments_scalability}
	\subfloat[Linear MAP problems]{
		\centering
		\includegraphics[width=0.45\linewidth]{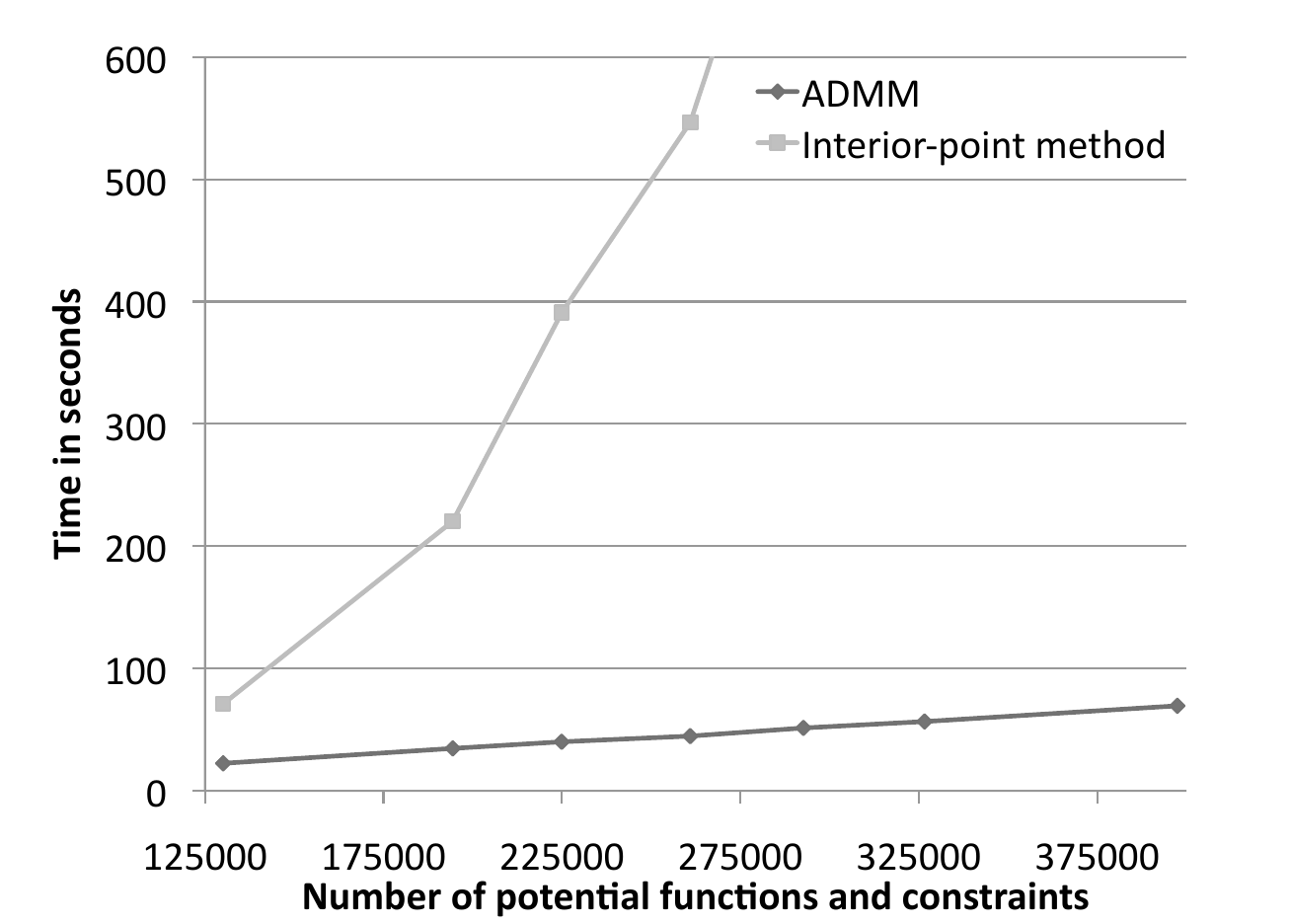}
		\label{fig:linear}
	}
	\;
	\subfloat[Quadratic MAP problems]{
		\centering
		\includegraphics[width=0.45\linewidth]{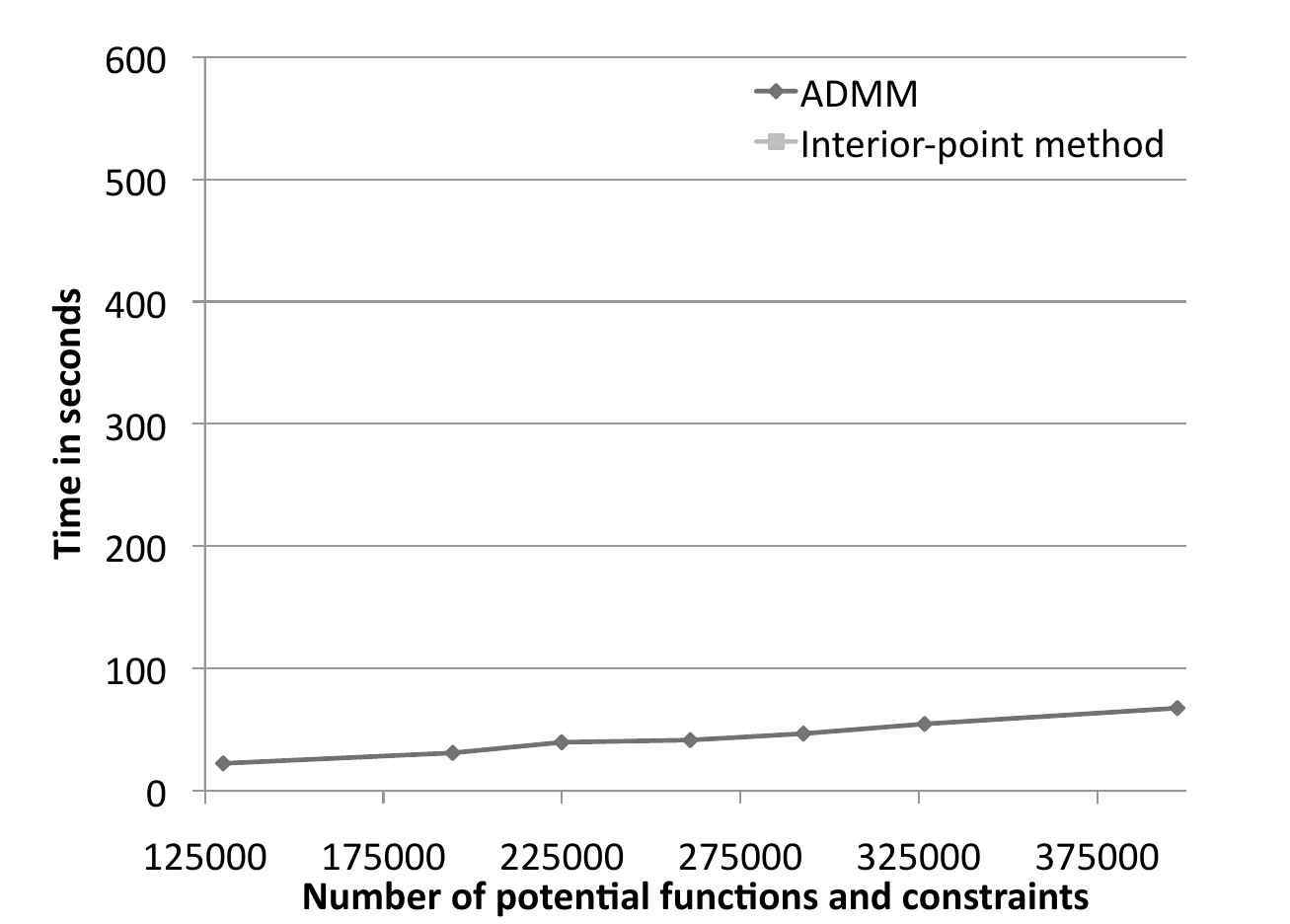}
		\label{fig:quad}
	}
	\\
	\subfloat[Linear MAP problems (log scale)]{
		\centering
		\includegraphics[width=0.45\linewidth]{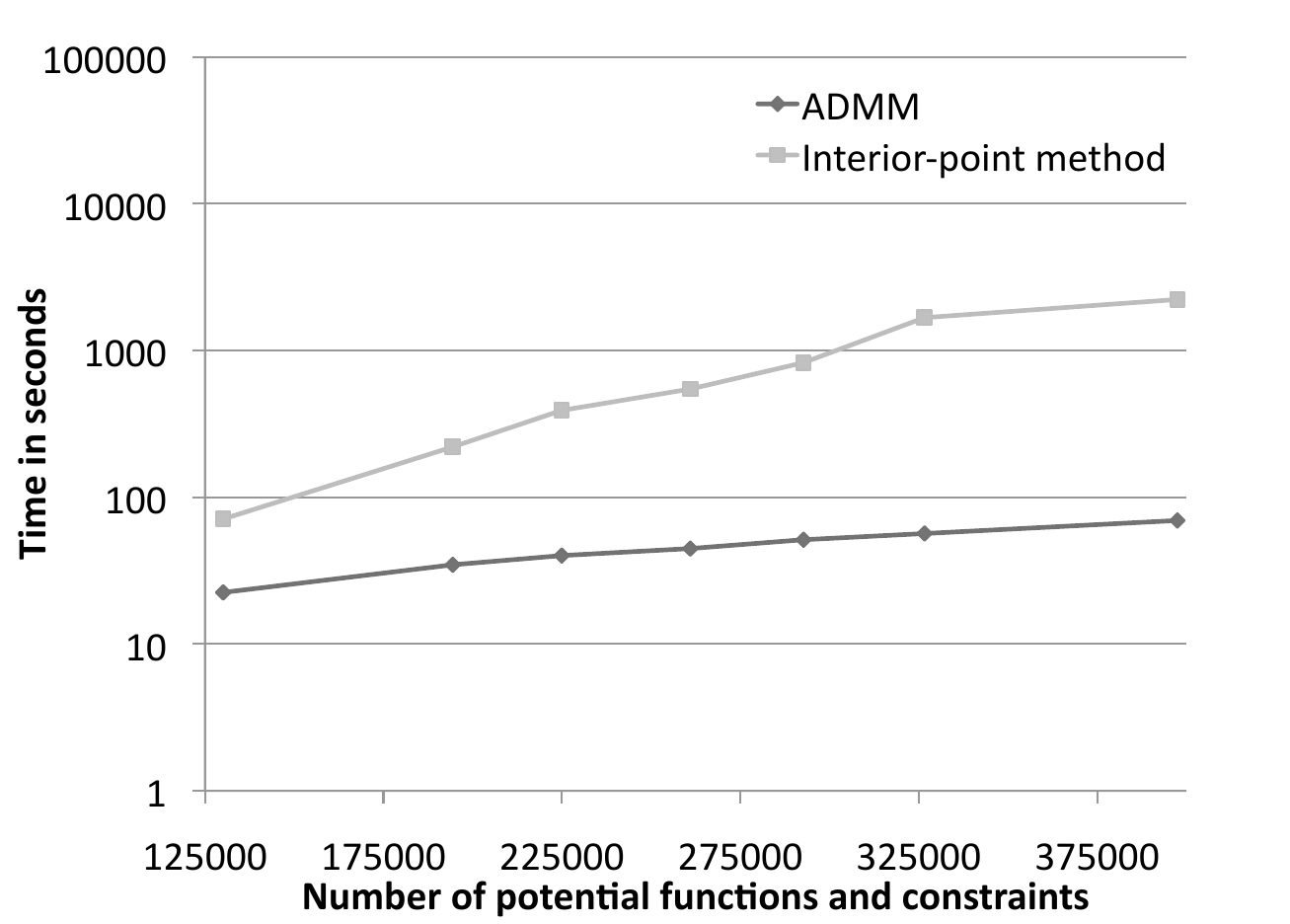}
		\label{fig:linear-log}
	}
	\;
	\subfloat[Quadratic MAP problems (log scale)]{
		\centering
		\includegraphics[width=0.45\linewidth]{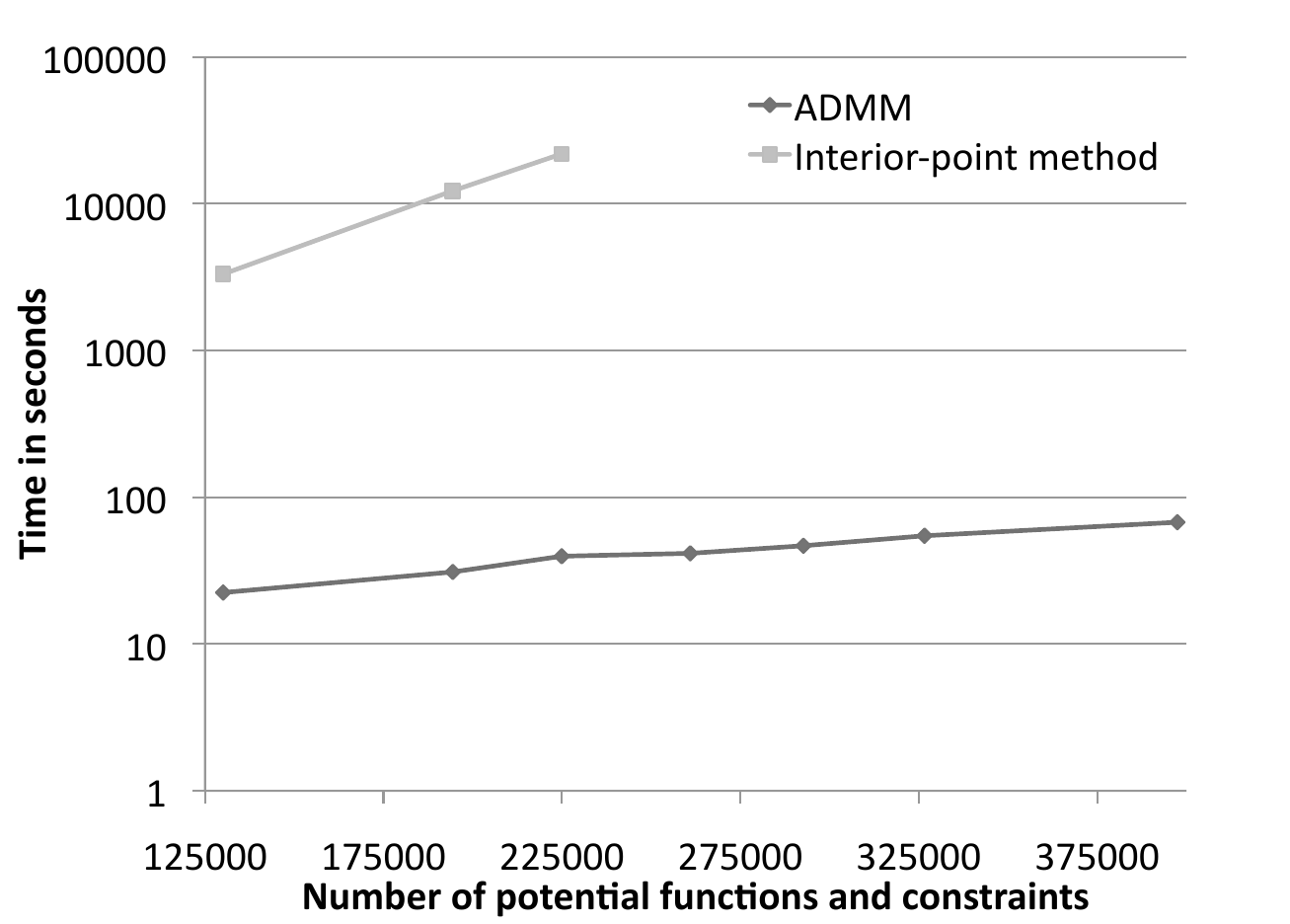}
		\label{fig:quad-log}
	}
	\caption{Average running times to find a MAP state for HL-MRFs.}
	\label{fig:all}
\end{figure}

We first evaluate the scalability of ADMM when solving piecewise-linear MAP problems and compare with MOSEK's
interior-point method.
Figures~\ref{fig:linear} (normal scale) and~\ref{fig:linear-log} (log scale) show the results.
The running time of the IPM quickly explodes as the problem size increases.
The IPM's average running time on the largest problem is about 2,200 seconds (37 minutes).
This result demonstrates the limited scalability of the interior-point method.
In contrast, ADMM displays excellent scalability.
The average running time on the largest problem is about 70 seconds.
Further, the running time appears to grow linearly in the number of potential functions and constraints
in the HL-MRF, i.e., the number of subproblems that must be solved at each iteration.
The line of best fit for all runs on all sizes has a coefficient of determination $R^2 = 0.9972$.
Combined with Figure~\ref{fig:linear}, this shows that ADMM scales linearly with increasing problem size in this experiment.
We emphasize that the implementation of ADMM is research code written in Java and the IPM is a
commercial package compiled to native machine code.

We then evaluate the scalability of ADMM when solving piecewise-quadratic MAP problem and again compare with MOSEK.
Figures~\ref{fig:quad} (normal scale) and~\ref{fig:quad-log} (log scale) show the results.
Again, the running time of the interior-point method quickly explodes.
We can only test it on the three smallest problems, the largest of which took an average of about
21k seconds to solve (over 6 hours).
ADMM again scales linearly to the problem ($R^2 = 0.9854$).
It is just as fast for quadratic problems as linear ones, taking average of about 70 seconds on the largest problem.

One of the advantages of IPMs is great numerical stability and accuracy.
Consensus optimization, which treats both objective terms and constraints as subproblems,
often returns solutions that are only optimal and feasible to moderate precision for non-trivially constrained
problems \citep{boyd:book11}.
Although this is often acceptable, we quantify the mix of infeasibility
and suboptimality by repairing the infeasibility and measuring the resulting total suboptimality.
We first project the solutions returned by consensus optimization onto the feasible region, which
took a negligible amount of computational time.
Let $p_\text{ADMM}$ be the value of the objective in Problem~(\ref{eq:modmap})
at such a point and let $p_\text{IPM}$ be the value of the objective at the solution returned by the IPM.
Then the relative error on that problem is $(p_\text{ADMM} - p_\text{IPM}) / p_\text{IPM}$.
The relative error was consistently small; it varied between 0.2\% and 0.4\%, and did not trend upward as the problem size increased.
This shows that ADMM was accurate, in addition to being much more scalable.



\section{Weight Learning}
\label{sec:learning}

In this section we present three weight learning methods for HL-MRFs, each with a different objective function.
The first method approximately maximizes the likelihood of the training data.
The second method maximizes the pseudolikelihood.
The third method finds a large-margin solution, preferring weights that discriminate the ground truth from other nearby states.
Since weights are often shared among many potentials defined by a template,
such as all the groundings of a PSL rule, we describe these learning algorithms in terms of templated HL-MRFs.
We introduce some necessary notation for HL-MRF templates.
Let $\Temps = (\temp_1,\dots,\temp_\ntemp)$ denote a vector of templates with associated weights
$\PParam = (\Param_1,\dots,\Param_{\ntemp})$.
We partition the potentials by their associated templates and let $\temp_\itemp$ also denote the set of indices
of the potentials defined by that template.
So, $\ipot \in \temp_{\itemp}$ is a shorthand for saying that the potential $\pot_\ipot(\yy, \xx)$ was defined by template
$\temp_\itemp$.
Then, we refer to the sum of the potentials defined by a template as
\begin{equation}
\Pot_{\itemp}(\yy,\xx) = \sum_{\ipot \in \temp_{\itemp}} \pot_\ipot(\yy,\xx)~.
\end{equation}
In the defined HL-MRF, the weight of the $\ipot$-th hinge-loss potential is set to the weight of the template from
which it was derived, i.e., $ \param_\ipot = \Param_{\itemp}$, for each  $\ipot \in \temp_{\itemp}$.
Equivalently, we can rewrite the hinge-loss energy function as
\begin{equation}
\energy(\yy, \xx) = \PParam^\top\PPot(\yy, \xx)~,
\end{equation}
where $\PPot(\yy, \xx) = (\Pot_1(\yy,\xx),\dots,\Pot_\ntemp(\yy,\xx))$.

\subsection{Structured Perceptron and Approximate Maximum Likelihood Estimation}
\label{sec:maxlikelihood}

The canonical approach for learning parameters $\PParam$ is to maximize the log-likelihood of training data.
The partial derivative of the log-likelihood with respect to a parameter $\Param_q$ is
\begin{equation}
\frac{\partial \log \prob(\yy | \xx)}{\partial \Param_\itemp} 
= \E_{\PParam} \left[ \Pot_\itemp(\yy, \xx) \right]
- \Pot_\itemp(\yy, \xx),
\end{equation}
where $\E_{\PParam}$ is the expectation under the distribution defined by $\PParam$.
For a smoother ascent, it is often helpful to divide the $\itemp$-th component of
the gradient by the number of groundings $|\temp_\itemp|$ of the $\itemp$-th template \citep{lowd:pkdd07}, which we do in our experiments.
Computing the expectation is intractable, so we use a common approximation \citep[e.g.,][]{collins:emnlp02, singla:aaai05, poon:uai11}: 
the values of the potentials at the most probable setting of $\yy$ with the current parameters, i.e., a MAP state.
Using a MAP state makes this learning approach a structured variant of voted perceptron \citep{collins:emnlp02}, and we expect it to do best when the space of explored distributions has relatively low entropy.
Following voted perceptron, we take steps of fixed length in the direction of the gradient, then average the points after all steps.
Any step that is outside the feasible region is projected back before continuing.

\subsection{Maximum Pseudolikelihood Estimation}
\label{sec:maxpseudo}

An alternative to structured perceptron is \emph{maximum-pseudolikelihood estimation} (MPLE) \citep{besag:jsr75},
which maximizes the likelihood of each variable conditioned on all other variables, i.e.,
\begin{align}
\prob^*(\yy | \xx)
	&= \prod_{i=1}^{\ny} \prob^*(\y_i | \MB(\y_i), \xx) \\
	&= \prod_{i=1}^{\ny} \frac{1}{\partition_\iy(\PParam, \yy, \xx)} \exp\left[ -\energy^i(\y_i,\yy,\xx) \right] ; \\
\partition_\iy(\PParam,\yy, \xx)
	&= \int_{\y_i}\exp\left[ -\energy^i(\y_i,\yy,\xx) \right] ; \\
\energy^i(\y_i,\yy,\xx)
	&= \sum_{\ipot : \iy \in \pot_\ipot} \param_\ipot \pot_j\left( \{\y_i \cup \yy\rem{i}\},\xx \right) .
\end{align}
Here, $\iy \in \pot_\ipot$ means that $\y_i$ is involved in $\pot_j$, and $\MB(\y_\iy)$ denotes the \emph{Markov blanket} of $\y_i$---that is, the set of variables that co-occur with $\y_\iy$ in any potential function. The partial derivative of the log-pseudolikelihood with respect to $\Param_\itemp$ is
\begin{equation}
\frac{\partial\log\prob^*(\yy | \xx)}{\partial \Param_\itemp}
	= \sum_{i=1}^{\ny} \E_{\y_i | \MB} \left[ \sum_{\ipot \in \temp_\itemp : \iy \in \pot_\ipot} \pot_\ipot(\yy,\xx) \right] - \Pot_\itemp(\yy,\xx)~.
\end{equation}
Computing the pseudolikelihood gradient does not require joint inference and takes time linear in the size of $\yy$.
However, the integral in the above expectation does not readily admit a closed-form antiderivative,
so we approximate the expectation.
When a variable is unconstrained, the domain of integration is a one-dimensional interval on the real number line,
so Monte Carlo integration quickly converges to an accurate estimate of the expectation.

We can also apply MPLE when the constraints are not too interdependent.
For example, for linear equality constraints over disjoint groups of variables (e.g., variable sets that must sum to 1.0),
we can block-sample the constrained variables by sampling uniformly from a simplex.
These types of constraints are often used to represent categorical labels.
We can compute accurate estimates quickly because these blocks are typically low-dimensional.

\subsection{Large-Margin Estimation}
\label{sec:largemargin}

A different approach to learning drops the probabilistic interpretation of the model and
views HL-MRF\ inference as a prediction function.
Large-margin estimation (LME) shifts the goal of learning from producing accurate probabilistic models to instead producing accurate MAP predictions.
The learning task is then to find weights $\PParam$ that separate the ground truth from other nearby states by a large margin.
We describe in this section a large-margin method based on the cutting-plane approach for structural support vector machines \citep{joachims:ml09}.

The intuition behind large-margin structured prediction is that the ground-truth state should have energy lower than any alternate state by a large margin. 
In our setting, the output space is continuous, so we parameterize this margin criterion with a continuous loss function. 
For any valid output state $\altyy$, a large-margin solution should satisfy
\begin{equation}
\energy(\yy, \xx) \le \energy(\altyy, \xx) - \loss(\yy, \altyy),~~\forall \altyy,
\end{equation}
where the loss function $\loss(\yy, \altyy)$ measures the disagreement between a state $\altyy$ and the training label state $\yy$.
A common assumption is that the loss function decomposes over the prediction components, i.e.,
$\loss(\yy, \altyy) = \sum_\iy \loss(\y_\iy, \alty_\iy)$.
In this work, we use the $\ell_1$ distance as the loss function, so $\loss(\yy, \altyy) = \sum_\iy \| \y_\iy - \alty_\iy \|_1$.
Since we do not expect all problems to be perfectly separable, we relax the large-margin constraint with a penalized slack $\slack$. 
We obtain a convex learning objective for a large-margin solution
\begin{equation}
\begin{aligned}
\min_{\PParam \ge 0} ~ & ~ \frac{1}{2} ||\PParam||^{2} + C \slack \\
 \mathrm{s.t.} ~ & ~ \PParam^{\top} ( \Pot(\yy, \xx) - \Pot(\altyy, \xx) ) \le - \loss(\yy, \altyy) + \slack,~~\forall \altyy,
\end{aligned}
\end{equation}
where $\Pot(\yy, \xx) = (\Pot_1(\yy,\xx),\dots,\Pot_\ntemp(\yy,\xx))$ and $C > 0$ is a user-specified parameter.
This formulation is analogous to the margin-rescaling approach by \citet{joachims:ml09}. Though such a structured objective is natural and intuitive, its number of constraints is the cardinality of the output space, which here is infinite.
Following their approach, we optimize subject to the infinite constraint set using a \emph{cutting-plane algorithm}: 
we greedily grow a set $\constraints$ of constraints by iteratively adding the worst-violated constraint
given by a \emph{separation oracle}, then updating $\PParam$ subject to the current constraints.
The goal of the cutting-plane approach is to efficiently find the set of active constraints at the solution for the full objective, without having to enumerate the infinite inactive constraints. 
The worst-violated constraint is
\begin{equation}
\argmin_{\altyy} \PParam^{\top} \Pot(\altyy, \xx) - \loss(\yy, \altyy).
\end{equation}
The separation oracle performs loss-augmented inference by adding additional potentials to the HL-MRF.
For ground truth in $\{0,1\}$, these loss-augmenting potentials are also examples of hinge-losses,
and thus adding them simply creates an augmented HL-MRF.
The worst-violated constraint is then computed as standard inference on the loss-augmented HL-MRF.
However, ground truth values in the interior $(0,1)$ cause any distance-based loss to be concave,
which require the separation oracle to solve a non-convex objective.
In this case, we use the \emph{difference of convex functions algorithm}
\citep{an:aor05} to find a local optimum.
Since the concave portion of the loss-augmented inference objective pivots around the ground truth
value, the subgradients are 1 or $-1$, depending on whether the current value is greater than the ground
truth.
We simply choose an initial direction for interior labels by rounding, and flip the direction of the
subgradients for variables whose solution states are not in the interval corresponding to the subgradient
direction until convergence.

Given a set $\constraints$ of constraints, we solve the SVM objective as in the primal form
\begin{equation}
\begin{aligned}
\min_{\PParam \ge 0} ~ & ~ \frac{1}{2} ||\PParam||^{2} + C \slack \\
 \mathrm{s.t.} ~ & ~ \constraints.
\end{aligned}
\end{equation}
We then iteratively invoke the separation oracle to find the worst-violated constraint. 
If this new constraint is not violated, or its violation is within numerical tolerance, we have found the max-margin solution. 
Otherwise, we add the new constraint to $\constraints$, and repeat.

One fact of note is that the large-margin criterion always requires some slack for HL-MRFs with squared potentials.
Since the squared hinge potential is quadratic and the loss is linear, there always exists a small enough distance from
the ground truth such that an absolute (i.e., linear) distance is greater than the squared distance.
In these cases, the slack parameter trades off between the peakedness of the learned quadratic energy function
and the margin criterion.

\subsection{Evaluation of Learning}
\label{sec:experiments}

To demonstrate the flexibility and effectiveness of learning with HL-MRFs, we test them on four diverse
tasks: node labeling, link labeling, link prediction, and image completion.\footnote{Code is available at \url{https://github.com/stephenbach/bach-jmlr17-code}.}
Each of these experiments represents a problem domain that is best solved with structured-prediction approaches because their dependencies
are highly structural.
The experiments show that HL-MRFs perform as well as or better than canonical approaches. 

For these diverse tasks, we compare against a number of competing methods. 
For node and link labeling, we compare HL-MRFs to discrete Markov random fields (MRFs).
We construct them with Markov logic networks (MLNs) \citep{richardson:ml06}, which template discrete MRFs using logical rules
similarly to PSL.
We perform inference in discrete MRFs using Gibbs sampling,
and we find approximate MAP states during learning using the search algorithm MaxWalkSat \citep{richardson:ml06}.
For link prediction for preference prediction, a task that is inherently continuous and nontrivial to encode in discrete logic, 
we compare against Bayesian probabilistic matrix factorization (BPMF) \citep{salakhutdinov:icml08}. 
Finally, for image completion, we run the same experimental setup as \citet{poon:uai11} and compare against the results they
report, which include tests using sum product networks, deep belief networks \citep{hinton:science06},
and deep Boltzmann machines \citep{salakhutdinov:aistats09}.

We train HL-MRFs and discrete MRFs with all three learning methods: structured perceptron (SP), maximum
pseudolikelihood estimation(MPLE), and large-margin estimation (LME).
When appropriate, we evaluate statistical significance using a paired t-test with rejection threshold 0.01. 
We describe the HL-MRFs used for our experiments using the PSL rules that define them.
To investigate the differences between linear and squared potentials we use both in our experiments.
HL-MRF-L refers to a model with all linear potentials and HL-MRF-Q to one with all squared potentials.
When training with SP and MPLE, we use 100 gradient steps and a step size of 1.0 (unless otherwise noted), and we average the iterates as in voted perceptron.
For LME, we set $C=0.1$.
We experimented with various settings, but the scores of HL-MRFs and discrete MRFs were not sensitive to changes.

\subsubsection{Node Labeling}
\label{sec:document_classification}

When classifying documents, links between those documents---such as hyperlinks, citations, or shared authorship---provide extra signal
beyond the local features of individual documents.
Collectively predicting document classes with these links tends to improve accuracy \citep{sen:aimag08}. 
We classify documents in citation networks using data from the Cora and Citeseer scientific paper
repositories.
The Cora data set contains 2,708 papers in seven categories, and 5,429 directed citation links. 
The Citeseer data set contains 3,312 papers in six categories, and 4,591 directed citation links.
Let the predicate \texttt{Category/2} represent the category of each document and \texttt{Cites/2} represent a citation from
one document to another.

The prediction task is, given a set of seed documents whose labels are observed, to infer the remaining document classes by
propagating the seed information through the network. For each of 20 runs, we split the data sets 50/50 into training and testing
partitions, and seed half of each set. 
To predict discrete categories with HL-MRFs we predict the category with the highest predicted value.

We compare HL-MRFs to discrete MRFs on this task.
For prediction, we performed 2500 rounds of Gibbs sampling, 500 of which were discarded as burn-in.
We construct both using the same logical rules, which simply encode the tendency for a class to propagate across citations.
For each category \texttt{"C\_i"}, we have the following two rules, one for each direction of citation.
\begin{align*}
&\texttt{Category(A, "C\_i") \psland\ Cites(A, B) \pslimpliesright\ Category(B, "C\_i")}\\
&\texttt{Category(A, "C\_i") \psland\ Cites(B, A) \pslimpliesright\ Category(B, "C\_i")}
\end{align*}
We also constrain the atoms of the \texttt{Category/2} predicate to sum to 1.0 for a given document as follows.
\[
\texttt{Category(D, +C) = 1.0 .}
\]
Table~\ref{tab:classification} lists the results of this experiment.
HL-MRFs are the most accurate predictors on both data sets.
Both variants of HL-MRFs are also much faster than discrete MRFs.
See Table~\ref{tab:timing} for average inference times over five folds.

\begin{table}
\caption{Average accuracy of classification by HL-MRFs and discrete MRFs.
Scores statistically equivalent to the best scoring method are typed in bold.}
\label{tab:classification}
\begin{center}
\begin{tabular}{lccc}
\toprule
  & Citeseer & Cora\\
\midrule
HL-MRF-Q (SP) & \textbf{0.729} & \textbf{0.816}\\
HL-MRF-Q (MPLE) & \textbf{0.729} & \textbf{0.818}\\
HL-MRF-Q (LME) & 0.683 & 0.789\\
\addlinespace
HL-MRF-L (SP) & \textbf{0.724} & 0.802\\
HL-MRF-L (MPLE) & \textbf{0.729} & \textbf{0.808}\\
HL-MRF-L (LME) & 0.695 & 0.789\\
\addlinespace
MRF (SP) & 0.686 & 0.756\\
MRF (MPLE) & 0.715 & 0.797\\
MRF (LME) & 0.687 & 0.783\\
\bottomrule
\end{tabular}
\end{center}
\end{table}

\subsubsection{Link Labeling}
\label{sec:trust}

\begin{table}
\caption{Average area under ROC and precision-recall curves of social-trust prediction by HL-MRFs and discrete MRFs.
Scores statistically equivalent to the best scoring method by metric are typed in bold.}\label{tab:epinions}
\begin{center}
\begin{tabular}{lrrr}
\toprule
 & ROC & P-R (+) & P-R (-) \\
\midrule
HL-MRF-Q (SP) & \textbf{0.822} & \textbf{0.978} & \textbf{0.452}\\
HL-MRF-Q (MPLE) & \textbf{0.832} & \textbf{0.979} & \textbf{0.482}\\
HL-MRF-Q (LME) & \textbf{0.814} & \textbf{0.976} & \textbf{0.462}\\
\addlinespace
HL-MRF-L (SP) & 0.765 & 0.965 & 0.357\\
HL-MRF-L (MPLE) & 0.757 & 0.963 & 0.333\\
HL-MRF-L (LME) & 0.783 & 0.967 & \textbf{0.453}\\
\addlinespace
MRF (SP) & 0.655 & 0.942 & 0.270\\
MRF (MPLE) & 0.725 & 0.963 & 0.298\\
MRF (LME) & 0.795 & \textbf{0.973} & \textbf{0.441}\\
\bottomrule
\end{tabular}
\end{center}
\end{table}

An emerging problem in the analysis of online social networks is the task of inferring the level of trust between individuals.
Predicting the strength of trust relationships can provide useful information for viral marketing, recommendation engines, and
 internet security.
 HL-MRFs with linear potentials have been applied by \citet{huang:sbp13} to this task, showing superior results with models
 based on sociological theory.
 We reproduce their experimental setup using their sample of the signed Epinions trust network,
 orginally collected by \citet{richardson:iswc03}, in which users indicate whether
 they trust or distrust other users.
 We perform eight-fold cross-validation.
 In each fold, the prediction algorithm observes the entire unsigned social network and all but 1/8 of the trust ratings.
 We measure prediction accuracy on the held-out 1/8.
 The sampled network contains 2,000 users, with 8,675 signed links.
 Of these links, 7,974 are positive and only 701 are negative, making it a sparse prediction task.

We use a model based on the social theory of \emph{structural balance}, which suggests that social structures are governed by a
system that prefers triangles that are considered balanced.
Balanced triangles have an odd number of positive trust relationships; thus, considering all possible directions of links that
form a triad of users, there are sixteen logical implications of the following form.
\[
\texttt{Trusts(A,B) \psland\ Trusts(B,C) \pslimpliesright\ Trusts(A,C)}
\]
\citet{huang:sbp13} list all sixteen of these rules, a reciprocity rule, and a prior in their \emph{Balance-Recip} model,
which we omit to save space. 

Since we expect these structural implications to vary in accuracy, learning weights for these rules provides
better models.
Again, we use these rules to define HL-MRFs and discrete MRFs, and we train them using various learning algorithms.
For inference with discrete MRFs, we perform 5000 rounds of Gibbs sampling, of which the first 500 are burn-in.
We compute three metrics: the area under the receiver operating characteristic (ROC) curve, and the areas under the
precision-recall curves for positive trust and negative trust. 
On all three metrics, HL-MRFs with squared potentials score significantly higher. 
The differences among the learning methods for squared HL-MRFs are insignificant,
but the differences among the models is statistically significant for the ROC metric.
For area under the precision-recall curve for positive trust, discrete MRFs trained with LME are statistically tied with the best score,
and both HL-MRF-L and discrete MRFs trained with LME are statistically tied with the best area under the precision-recall curve for
negative trust. 
The results are listed in Table~\ref{tab:epinions}.

Though the random fold splits are not the same, using the same experimental setup,
\citet{huang:sbp13} also scored the precision-recall area for negative trust of standard trust prediction algorithms
EigenTrust \citep{kamvar:www03} and TidalTrust \citep{golbeck:thesis05}, which scored 0.131 and 0.130, respectively.
The logical models based on structural balance that we run here are significantly more accurate, and HL-MRFs more than
discrete MRFs.

In addition to comparing favorably with regard to predictive accuracy, inference in HL-MRFs is also much faster than in discrete
MRFs.
Table~\ref{tab:timing} lists average inference times on five folds of three prediction tasks: Cora, Citeseer, and Epinions.
This illustrates an important difference between performing structured prediction via convex inference
versus sampling in a discrete prediction space: convex inference can be much faster.

\begin{table}
\caption{Average inference times (reported in seconds) of single-threaded HL-MRFs and discrete MRFs.
}\label{tab:timing}
\begin{center}
\begin{tabular}{lrrr}
\toprule
 & Citeseer & Cora & Epinions \\
\midrule
HL-MRF-Q & 0.42 & 0.70 & 0.32 \\
HL-MRF-L  & 0.46 & 0.50 & 0.28 \\
MRF & 110.96 & 184.32 & 212.36 \\
\bottomrule
\end{tabular}
\end{center}
\end{table}

\subsubsection{Link Prediction}
\label{sec:jester}

Preference prediction is the task of inferring user attitudes (often quantified by ratings) toward a set of items. 
This problem is naturally structured, since a user's preferences are often interdependent, as are an item's ratings. 
\emph{Collaborative filtering} is the task of predicting unknown ratings using only a subset of observed ratings. 
Methods for this task range from simple nearest-neighbor classifiers to complex latent factor models.
More generally, this problem is an instance of link prediction, since the goal is to predict links indicating preference between
users and content.
Since preferences are ordered rather than Boolean, it is natural to represent them with the continuous variables of HL-MRFs,
with higher values indicating greater preference.
To illustrate the versatility of HL-MRFs, we design a simple, interpretable collaborative filtering model for predicting humor
preferences.
We test this model on the Jester dataset, a repository of ratings from 24,983 users on a set of 100 jokes \citep{goldberg:ir01}. 
Each joke is rated on a scale of $[-10,+10]$, which we normalize to $[0,1]$. 
We sample a random 2,000 users from the set of those who rated all 100 jokes, which we then split into 1,000 train
and 1,000 test users. 
From each train and test matrix, we sample a random 50\% to use as the observed features $\xx$;
the remaining ratings are treated as the variables $\yy$.

Our HL-MRF model uses an item-item similarity rule:
\[
\texttt{SimRating(J1, J2) \psland\ Likes(U, J1) \pslimpliesright\ Likes(U, J2)}
\]
where \texttt{J1} and \texttt{J2} are jokes and \texttt{U} is a user;
the predicate \texttt{Likes/2} indicates the degree of preference (i.e., rating value); 
and \texttt{SimRating/2} is a closed predicate that measures the mean-adjusted cosine similarity between
the observed ratings of two jokes. 
We also include the following rules to enforce that \texttt{Likes(U,J)} concentrates around the observed average rating of user
\texttt{U} (represented with the predicate \texttt{AvgUserRating/1}) and item \texttt{J}
(represented with the predicate \texttt{AvgJokeRating/1}), and the global average (represented with
the predicate \texttt{AvgRating/1}).
\begin{align*}
&\texttt{AvgUserRating(U) \pslimpliesright\ Likes(U, J)} \\
&\texttt{Likes(U, J) \pslimpliesright\ AvgUserRating(U)} \\
&\texttt{AvgJokeRating(J) \pslimpliesright\ Likes(U, J)} \\
&\texttt{Likes(U, J) \pslimpliesright\ AvgJokeRating(J)} \\
&\texttt{AvgRating("constant") \pslimpliesright\ Likes(U, J)} \\
&\texttt{Likes(U, J) \pslimpliesright\ AvgRating("constant")}
\end{align*}
The atom \texttt{AvgRating("constant")} takes a placeholder constant as an argument, since there is only one
grounding of it for the entire HL-MRF.
Again, all three of these predicates are closed and computed using averages of observed ratings.
In all cases, the observed ratings are taken only from the training data for learning (to avoid leaking
information about the test data) and only from the test data during testing.

We compare our HL-MRF model to a canonical latent factor model,
\emph{Bayesian probabilistic matrix factorization} (BPMF) \citep{salakhutdinov:icml08}.
BPMF is a fully Bayesian treatment and is therefore considered ``parameter-free;" the only parameter that must be
specified is the rank of the decomposition.
Based on settings used by \citet{xiong:sdm10}, we set the rank of the decomposition to 30 and use 100 iterations of burn in
and 100 iterations of sampling.
For our experiments, we use the code of \citet{xiong:sdm10}.
Since BPMF does not train a model, we allow BPMF to use all of the training matrix during the prediction phase.

Table~\ref{tab:jester} lists the normalized mean squared error (NMSE) and normalized mean absolute error (NMAE), averaged over 10 random splits. Though BPMF produces the best scores, the improvement over HL-MRF-L (LME) is not significant in NMAE. 

\begin{table}
\caption{Normalized mean squared/absolute errors (NMSE/NMAE) for preference prediction using the Jester dataset.
The lowest errors are typed in bold.
}
\label{tab:jester}
\begin{center}
\begin{tabular}{lcc}
\toprule
& NMSE & NMAE \\
\midrule
HL-MRF-Q (SP) & 0.0554 & 0.1974 \\
HL-MRF-Q (MPLE) & 0.0549 & 0.1953 \\
HL-MRF-Q (LME) & 0.0738 & 0.2297 \\
\addlinespace
HL-MRF-L (SP) & 0.0578 & 0.2021 \\
HL-MRF-L (MPLE) & 0.0535 & 0.1885 \\
HL-MRF-L (LME) & 0.0544 & \textbf{0.1875} \\
\addlinespace
BPMF & \textbf{0.0501} & \textbf{0.1832}\\
\bottomrule
\end{tabular}
\end{center}
\end{table}

\subsubsection{Image Completion}
\label{sec:faces}

\begin{table}
\caption{Mean squared errors per pixel for image completion. HL-MRFs produce the most accurate completions on the Caltech101 and the left-half Olivetti faces, and only sum-product networks produce better completions on Olivetti bottom-half faces. Scores for other methods are reported in \citet{poon:uai11}.}\label{tab:vision}
\begin{center}
\begin{tabular}{lcccccc}
\toprule
 & HL-MRF-Q (SP) & SPN & DBM & DBN & PCA & NN\\
\midrule
Caltech-Left&  1741 & 1815 & 2998 & 4960 & 2851 & 2327 \\
Caltech-Bottom & 1910 & 1924 & 2656 & 3447 & 1944 & 2575\\
Olivetti-Left & 927 &  942 & 1866 & 2386 & 1076 & 1527 \\
Olivetti-Bottom & 1226 & 918 & 2401 & 1931 & 1265 & 1793\\
\bottomrule
\end{tabular}
\end{center}
\end{table}

\begin{figure}
\begin{center}
\includegraphics[width=2.5in]{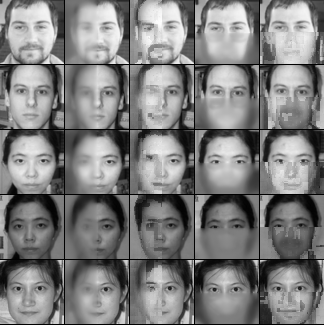}
\hspace{.1in}
\includegraphics[width=2.5in]{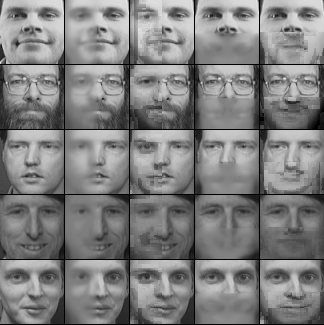}
\end{center}
\caption{Example results on image completion of Caltech101 (left) and Olivetti (right) faces. From left to right in each column:
(1) true face, left side predictions by (2) HL-MRFs and (3) SPNs, and bottom half predictions by (4) HL-MRFs and (5) SPNs.
SPN completions are downloaded from \citet{poon:uai11}.}
\label{fig:faces}
\end{figure}

Digital image completion requires models that understand how pixels relate to each other, such that when some pixels are
unobserved, the model can infer their values from parts of the image that are observed.
We construct pixel-grid HL-MRFs for image completion.
We test these models using the experimental setup of \citet{poon:uai11}:
we reconstruct images from the Olivetti face data set and the Caltech101 face category.
The Olivetti data set contains 400 images, 64 pixels wide and tall, and the Caltech101 face category contains 435 examples of faces, 
which we crop to the center 64 by 64 patch, as was done by \citet{poon:uai11}.
Following their experimental setup, we hold out the last fifty images and predict either the left half of the image or the bottom half. 

The HL-MRFs in this experiment are much more complex than the ones in our other experiments because we allow each pixel to
have its own weight for the following rules, which encode agreement or disagreement between neighboring pixels:
\begin{align*}
&\texttt{Bright("P\_ij", I) \psland\ North("P\_ij", Q) \pslimpliesright\ Bright(Q, I)} \\
&\texttt{Bright("P\_ij", I) \psland\ North("P\_ij", Q) \pslimpliesright\ !Bright(Q, I)} \\
&\texttt{!Bright("P\_ij", I) \psland\ North("P\_ij", Q) \pslimpliesright\ Bright(Q, I)} \\
&\texttt{!Bright("P\_ij", I) \psland\ North("P\_ij", Q) \pslimpliesright\ !Bright(Q, I)}
\end{align*}
where \texttt{Bright("P\_ij", I)} is the normalized brightness of pixel \texttt{"P\_ij"} in image \texttt{I},
and \texttt{North("P\_ij", Q)} indicates that \texttt{Q} is the north neighbor of \texttt{"P\_ij"}.
We similarly include analogous rules for the south, east, and west neighbors, as well as the pixels
mirrored across the horizontal and vertical axes.
This setup results in up to 24 rules per pixel, (boundary pixels may not have north, south, east, or west neighbors)
which, in a 64 by 64 image, produces 80,896 PSL rules.

We train these HL-MRFs using SP with a 5.0 step size on the first 200 images of each data set and test on the last fifty.
For training, we maximize the data log-likelihood of uniformly random held-out pixels for each training image, allowing for
generalization throughout the image.
Table~\ref{tab:vision} lists our results and others reported by \citet{poon:uai11} for sum-product networks (SPN),
deep Boltzmann machines (DBM), deep belief networks (DBN), principal component analysis (PCA), and nearest neighbor (NN).
HL-MRFs produce the best mean squared error on the left- and bottom-half settings for the Caltech101 set and the left-half setting
in the Olivetti set.
Only sum product networks produce lower error on the Olivetti bottom-half faces. 
Some reconstructed faces are displayed in Figure~\ref{fig:faces}, where the shallow, pixel-based HL-MRFs produce comparably
convincing images to sum-product networks, especially in the left-half setting, where HL-MRFs can learn which pixels are likely to
mimic their horizontal mirror.
While neither method is particularly good at reconstructing the bottom half of faces, the qualitative difference between the deep
SPN and the shallow HL-MRF completions is that SPNs seem to hallucinate different faces, often with some artifacts,
while HL-MRFs predict blurry shapes roughly the same pixel intensity as the observed, top half of the face. 
The tendency to better match pixel intensity helps HL-MRFs score better quantitatively on the Caltech101 faces, where the lighting
conditions are more varied than in Olivetti faces.

Training and predicting with these HL-MRFs takes little time.
In our experiments, training each model takes about 45 minutes on a 12-core machine, while predicting takes under a second
per image.
While \citet{poon:uai11} report faster training with SPNs, both HL-MRFs and SPNs clearly belong to a class of faster models when
compared to DBNs and DBMs, which can take days to train on modern hardware.



\section{Related Work}
\label{sec:related}

Researchers in artificial intelligence and machine learning have long been interested in predicting
interdependent unknowns using structural dependencies.
Some of the earliest work in this area is inductive logic programming (ILP) \citep{muggleton:lp94}, in
which structural dependencies are described with first-order logic.
Using first-order logic has several advantages.
First, it can capture many types of dependencies among variables, such as correlations, anti-correlations,
and implications.
Second, it can compactly specify dependencies that hold across many different sets of propositions by
using variables as wildcards that match entities in the data.
These features enable the construction of intuitive, general-purpose models that are easily applicable
or adapted to different domains.
Inference for ILP finds the propositions that satisfy a query, consistent with a relational knowledge base.
However, ILP is limited by its difficulty in coping with uncertainty.
Standard ILP approaches only model dependencies which hold universally, and such dependencies are
rare in real-world data.

Another broad area of research, probabilistic methods, directly models uncertainty over unknowns.
Probabilistic graphical models (PGMs) \citep{koller:book09} are a family of formalisms for
specifying joint distributions over interdependent unknowns through graphical structures.
The graphical structure of a PGM generally represents conditional independence relationships among
random variables.
Explicitly representing conditional independence relationships allows a distribution to be more compactly
parametrized.
For example, in the worst case, a discrete distribution could be represented by an exponentially large table
over joint assignments to the random variables.
However, describing the distribution in smaller, conditionally independent pieces can be much more
compact.
Similar benefits apply to continuous distributions.
Algorithms for probabilistic inference and learning can also operate over the conditionally independent
pieces described by the graph structure.
They are therefore straightforward to apply to a wide variety of distributions.
Categories of PGMs include Markov random fields (MRFs), Bayesian networks (BNs),
and dependency networks (DNs).
Constructing PGMs often requires careful design, and models are usually constructed for single tasks and data sets.

More recently, researchers have sought to combine the advantages of relational and probabilistic
approaches, creating the field of {\em statistical relational learning} (SRL) \citep{getoor:book07}.
SRL techniques build probabilistic models of relational data, i.e., data composed of entities and
relationships connecting them.
Relational data is most often described using a relational calculus, but SRL techniques are also
equally applicable to similar categories of data that go by other names, such as graph data or network data.
Modeling relational data is inherently complicated by the large number of interconnected and overlapping
structural dependencies that are typically present.
This complication has motivated two directions of work.
The first direction is algorithmic, seeking inference and learning methods that scale up to high dimensional models.
The other direction is both user-oriented and---as a growing body of evidence shows---supported by
learning theory, seeking formalisms for compactly specifying entire groups of dependencies in the model
that share both form and parameters.
Specifying these grouped dependencies, often in the form of templates via a domain-specific language,
is convenient for users.
Most often in relational data the structural dependencies hold without
regard to the identities of entities, instead being induced by an entity's class (or classes) and the
structure of its relationships with other entities.
Therefore, many SRL models and languages give users the ability to specify dependencies in this abstract form and ground out
models over specific data sets based on these definitions.
In addition to convenience, recent work in learning theory says that repeated dependencies with tied parameters
can be the key to generalizing from a few---or even one---large, structured training example(s) \citep{london:jmlr16}.

A related field to SRL is {\em structured prediction} (SP) \citep{bakir:book07, nowozin:book16}, which generalizes the
tasks of classification and regression to the task of predicting structured objects.
The loss function used during learning is generalized to a task-appropriate loss function
that scores disagreement between predictions and the true structures.
Often, models for structured prediction take the form of energy functions that are linear in their
parameters.
Therefore, prediction with such models is equivalent to MAP inference for MRFs.
A distinct branch of SP is learn-to-search methods, in which the problem is decomposed into a series of one-dimension
prediction problems.
The challenge is to learn a good order in which to predict the components of the structure, so that each one-dimension
prediction problem can be conditioned on the most useful information.
Examples of learn-to-search methods include incremental structured perceptron \citep{collins:acl04}, SEARN \citep{daume:ml09},
DAgger \citep{ross:aistats11}, and AggreVaTe \citep{ross:arxiv14}.

In this paper we focus on SP methods that perform joint prediction directly.
Better understanding the differences and relative advantages of joint-prediction methods and learn-to-search methods
is an important direction for future work.
In the rest of this section we survey models and domain-specific languages for SP and SRL (Section~\ref{sec:related_models}),
inference methods (Section~\ref{sec:related_inference}), and learning methods (Section~\ref{sec:related_learning}).

\subsection{Models and Languages}
\label{sec:related_models}

SP and SRL encompass many approaches.
One broad area of work---of which PSL is a part---uses first-order logic and other relational formalisms
to specify templates for PGMs.
Probabilistic relational models \citep{friedman:ijcai99} define templates for BNs in terms of a
database schema, and they can be grounded out over instances of that schema to create BNs.
Relational dependency networks \citep{neville:jmlr07} template RNs using structured query language
(SQL) queries over a relational schema.
Markov logic networks (MLNs) \citep{richardson:ml06} use first-order logic to define Boolean MRFs.
Each logical clause in a first-order knowledge base is a template for a set of potentials when the MLN
is grounded out over a set of propositions.
Whether each proposition is true is a Boolean random variable, and the potential has a value of one when
the corresponding ground clause is satisfied by the propositions and zero when it is not.
(MLNs are formulated such that higher values of the energy function are more probable.)
Clauses can either be weighted, in which case the potential has the weight of the clause that templated it,
or unweighted, in which case in must hold universally, as in ILP.
In these ways, MLNs are similar to PSL.
Whereas MLNs are defined over Boolean variables, PSL is a templating language for HL-MRFs, which are
defined over continuous variables.
However, these continuous variables can be used to model discrete quantities.
See Section~\ref{sec:logic} for more information on the relationships between HL-MRFs and discrete MRFs,
and Section~\ref{sec:experiments} for empirical comparisons between the two.
As we show, HL-MRFs and PSL scale much better while retaining the rich expressivity and accuracy of their discrete counterparts.
In addition, HL-MRFs and PSL can reason directly about continuous data.

PSL is part of a broad family of probabilistic programming languages \citep{gordon:icse14}.
The goals of probabilistic programming and SRL often overlap.
Probabilistic programming seeks to make constructing probabilistic models easy for the end user,
and separate model specification from the development of inference and learning algorithms.
If algorithms can be developed for the entire space of models covered by a language, then it is easy for users to experiment
with including and excluding different model components.
It also makes it easy for existing models to benefit from improved algorithms.
Separation of model specification and algorithms is also useful in SRL for the same reasons.
In this paper we emphasize designing algorithms that are flexible enough to support the full class of HL-MRFs.
Examples of probabilistic programming languages include
IBAL \citep{pfeffer:ijcai01},
BLOG \citep{milch:ijcai05},
Markov logic \citep{richardson:ml06},
ProbLog \citep{deraedt:ijcai07},
Church \citep{goodman:uai08},
Figaro \citep{pfeffer:techreport09},
FACTORIE \citep{mccallum:nips09},
Anglican \citep{wood:aistats14},
and Edward \citep{tran:arxiv16}.

Other formalisms have also been proposed for probabilistic reasoning over continuous domains and other domains equipped with semirings.
Hybrid Markov logic networks \citep{wang:aaai08} mix discrete and continuous variables.
In addition to the dependencies over discrete variables supported by MLNs, they support soft equality constraints between two variables of the same form as those defined by squared arithmetic rules in PSL, as well as linear potentials of the form $\y_1 - \y_2$ for a soft inequality constraint $\y_1 > \y_2$.
Inference in hybrid MLNs is intractable.
\citet{wang:aaai08} propose a random walk algorithm for approximate MAP inference.
Another related formalism is aProbLog \citep{kimmig:aaai11}, which generalizes ProbLog to allow clauses to be annotated with elements from a semiring, generalizing ProbLog's support for clauses annotated with probabilities.
Many common inference tasks can be generalized from this perspective as algebraic model counting \citep{kimmig:jal16}.
The PITA system \citep{riguzzi:iclp11} for probabilistic logic programming can also be viewe as implementing inference over various semirings.

\subsection{Inference}
\label{sec:related_inference}

Whether viewed as MAP inference for an MRF or SP without probabilistic semantics, searching over a
structured space to find the optimal prediction is an important but difficult task.
It is NP-hard in general \citep{shimony:ai94}, so much work has focused on approximations
and identifying classes of problems for which it is tractable.
A well-studied approximation technique is local consistency relaxation (LCR) \citep{wainwright:book08}.
Inference is first viewed as an equivalent optimization over the realizable expected values of the potentials,
called the marginal polytope.
When the variables are discrete and each potential is an indicator that a subset of variables is in a certain state,
this optimization becomes a linear program.
Each variable in the program is the marginal probability that a variable is a particular state or the variables
associated with a potential are in a particular joint state.
The marginal polytope is then the set of marginal probabilities that are globally consistent.
The number of linear constraints required to define the marginal polytope is exponential in the size of the
problem, however, so the linear program has to be relaxed in order to be tractable.
In a local consistency relaxation, the marginal polytope is relaxed to the local polytope, in which the
marginals over variables and potential states are only locally consistent in the sense that
each marginal over potential states sums to the marginal distributions over the associated variables.

A large body of work has focused on solving the LCR objective quickly.
Typically, off-the-shelf convex optimization methods do not scale well for large graphical models
and structured predictors \citep{yanover:jmlr06}, so a large branch of research has investigated highly
scalable message-passing algorithms.
One approach is dual decomposition (DD) \citep{sontag:optchapter11}, which solves a problem
dual to the LCR objective.
Many DD algorithms use coordinate descent, such as TRW-S \citep{komolgorov:pami06},
MSD \citep{werner:pami07}, MPLP \citep{globerson:nips07}, and ADLP \citep{meshi:ecml11}.
Other DD algorithms use subgradient-based approaches \citep[e.g.,][]{jojic:icml10, komodakis:pami11, schwing:nips12}.

Another approach to solving the LCR objective uses message-passing algorithms to solve the problem directly in
its primal form.
One well-known algorithm is that of \citet{ravikumar:jmlr10}, which uses proximal optimization, a
general approach that iteratively improves the solution by searching for nearby improvements.
The authors also provide rounding guarantees for when the relaxed solution is integral, i.e., the relaxation
is tight, allowing the algorithm to converge faster.
Another message-passing algorithm that solves the primal objective is AD$^3$ \citep{martins:jmlr15},
which uses the alternating direction method of multipliers (ADMM).
AD$^3$ optimizes objective~(\ref{eq:localvariational}) for binary, pairwise MRFs and supports the
addition of certain deterministic constraints on the variables.
A third example of a primal message-passing algorithm is APLP \citep{meshi:ecml11}, which is the primal
analog of ADLP.
Like AD$^3$, it uses ADMM to optimize the objective.

Other approaches to approximate inference include tighter linear programming relaxations \citep{sontag:uai08, sontag:uai12}.
These tighter relaxations enforce local consistency on variable subsets that are larger than individual
variables, which makes them {\em higher-order local consistency relaxations}.
\citet{mezuman:uai13} developed techniques for special cases of higher-order relaxations, such as when the MRF contains cardinality
potentials, in which the probability of a configuration depends on the number of variables in a particular state.
Researchers have also explored nonlinear convex programming relaxations, e.g., \citet{ravikumar:icml06} and 
\citet{kumar:cvpr06}.

Previous analyses have identified particular subclasses whose local consistency
relaxations are tight, i.e., the maximum of the relaxed program is exactly the maximum of the original problem. 
These special classes include graphical models with tree-structured dependencies, models with submodular potential functions,  models encoding bipartite matching problems, and those with \emph{nand} potentials and perfect graph structures \citep{wainwright:book08,schrijver:book03,jebara:uai09,foulds:aistats11}.
Researchers have also studied performance guarantees of other subclasses of the first-order local consistency relaxation.
\citet{kleinberg:jacm02} and \citet{chekuri:discmath05} considered the metric labeling problem.
\citet{feldman:infotheory05} used the local consistency relaxation to decode binary linear codes.

In this paper we examine the classic problem of MAX SAT---finding a joint Boolean assignment to a set of propositions that
maximizes the sum of a set of weighted clauses that are satisfied---as an instance of SP.
Researchers have also considered approaches to solving MAX SAT other than the one one we study, the randomized
algorithm of \citet{goemans:discmath94}.
One line of work focusing on convex programming relaxations has obtained stronger rounding guarantees than
\citet{goemans:discmath94} by using nonlinear programming, e.g., \citet{asano:alg02} and references therein.
Other work does not use the probabilistic method but instead searches for discrete solutions directly, e.g.,
\citet{mills:jar00}, \citet{larrosa:ai08}, and \citet{choi:cp09}.
We note that one such approach, that of \citet{wah:ijcai97}, is essentially a type of DD formulated for MAX SAT.
A more recent approach blends convex programming and discrete search via mixed integer programming \citep{davies:sat13}.
Additionally, \citet{huynh:ecml09} introduced a linear programming relaxation for MLNs inspired by MAX SAT relaxations,
but the relaxation of general Markov logic provides no known guarantees on the quality of solutions.

Finally, \emph{lifted inference} takes advantage of symmetries in probability distributions to reduce the amount of work required for inference.
Some of the earliest approaches identified repeated dependency structures in PGMs to avoid repeated computations \citep{koller:uai97, pfeffer:uai99}.
Lifted inference has been widely applied in SRL because the templates that are commonly used to define PGMs often induce symmetries.
Various inference techniques for discrete MRFs have been extended to a lifted approach, including belief propagation \citep{jaimovich:uai07, singla:aaai08, kersting:uai09} and Gibbs sampling \citep{venugopal:nips12}.
Approaches to lifted convex optimization \citep{mladenov:aistats12} might be extended to HL-MRFs.
See \citet{desalvobraz:07}, \citet{kersting:ecai12}, and \citet{kimmig:ml15} for more information on lifted inference.

\subsection{Learning}
\label{sec:related_learning}

\citet{taskar:nips03} connected SP and PGMs by showing how to train MRFs
with large-margin estimation, a generalization of the large-margin objective for binary classification used
to train support vector machines \citep{vapnik:book00}.
Large-margin learning is a well-studied approach to train structured predictors because it directly incorporates the
structured loss function into a convex upper bound on the true objective: the regularized expected risk.
The learning objective is to find the parameters with smallest norm such that a linear combination of feature functions assign a better
score to the training data than all other possible predictions.
The amount by which the score of the correct prediction must exceed the score of other predictions is scaled using the structured
loss function.
The objective is therefore encoded as a norm minimization problem subject to many linear constraints,
one for each possible prediction in the structured space.

Structured SVMs \citep{tsochantaridis:jmlr05} extend large-margin estimation to a broad class of structured
predictors and admit a tractable cutting-plane learning algorithm.
This algorithm will terminate in a number of iterations linear in the size of the problem,
and so the computational challenge of large-margin learning for structured prediction comes down to the
task of finding the most violated constraint in the learning objective.
This can be accomplished by optimizing the energy function plus the loss function.
In other words, the task is to find the structure that is the best combination of being favored by the energy function
but unfavored by the loss function.
Often, the loss function decomposes over the components of the prediction space,
so the combined energy function and loss function can often be viewed as simply
the energy function of another structured predictor that is equally challenging or easy to optimize,
such as when the space of structures is a set of discrete vectors and the loss function is the Hamming distance.

It is common during large-margin estimation that no setting of the parameters can predict all the training data without error.
In this case, the training data is said to not be separable, again generalizing the notion of linear separability in
the feature space from binary classification.
The solution to this problem is to add slack variables to the constraints that require the training data to be assigned
the best score.
The magnitude of the slack variables are penalized in the learning objective, so estimation must trade off between
the norm of the parameters and violating the constraints.
\citet{joachims:ml09} extend this formulation to a ``one slack'' formulation, in which a single slack variable is used
for all the constraints across all training examples, which is more efficient.
We use this framework for large-margin estimation for HL-MRFs in Section~\ref{sec:largemargin}.

The repeated inferences required for large-margin learning, one to find the most-violated constraint at each iteration, can
become computationally expensive.
Therefore researchers have explored speeding up learning by interleaving the inference problem with the learning problem.
In the cutting-plane formulation discussed above, the objective is equivalently a saddle-point problem, with the solution at
the minimum with respect to the parameters and the maximum with respect to the inference variables.
\citet{taskar:icml05} proposed dualizing the inner inference problem to form a joint minimization.
For SP problems with a tight duality gap, i.e., the dual problem has the same optimal value as the primal problem,
this approach leads to an equivalent, convex optimization that can be solved for all variables simultaneously.
In other words, the learning and most-violated constraint problems are solved simultaneously, greatly reducing training time.
For problems with non-tight duality gaps, e.g., MAP inference in general, discrete MRFs, \citet{meshi:icml10} showed that
the same principle can be applied by using approximate inference algorithms like dual decomposition to bound the
primal objective.

A related problem to parameter learning is \emph{structure learning}, i.e., identifying an accurate dependency structure for a model.
A common SRL approach is searching over the space of templates for PGMs.
For probabilistic relational models, \citet{friedman:ijcai99} learned structures
described in the vocabulary of relational schemas.
For models that are templated with first-order-logic-like languages, such as 
PSL and MLNs, these approaches take the form of rule learning.
Based on rule-learning techniques from inductive logic programming \citep[e.g.,][]{richards:aaai92, deraedt:ml96} a series of approaches have sought to learn MLN rules from relational data.
Initially, \citet{kok:icml05} learned rules by generating candidates and performing
a beam search to identify rules that improved a weighted pseudolikelihood objective.
Then, \citet{mihalkova:icml07} observed that the previous approach
generated candidate rules without regard to the data, so they introduced an approach that used
the data to guide the proposal of rules via \emph{relational pathfinding}.
\citet{kok:icml10} improved on that by first performing graph clustering to find
common \emph{motifs}, which are common subgraphs, to guide rule proposal.
They observed that modifying a rule set one clause at a time often got stuck in poor
local optima, and by using the motifs as refinement operators instead, they were
able to converge to better optima.
Other approaches to structure learning search directly over grounded PGMs, including $\ell_1$-regularized  pseudolikelihood maximization \citep{ravikumar:annalsofstats10} and grafting \citep{perkins:jmlr03, zhu:kdd10}.
These methods can all be extended to HL-MRFs and PSL.


\section{Conclusion}
\label{sec:conclusion}

In this paper we introduced HL-MRFs, a new class of probabilistic graphical models that unite and generalize several approaches to
modeling relational and structured data: Boolean logic, probabilistic graphical models, and fuzzy logic.
HL-MRFs can capture relaxed, probabilistic inference with Boolean logic and exact, probabilistic inference with fuzzy logic,
making them useful models for both discrete and continuous data.
HL-MRFs also generalize these inference techniques with additional expressivity, allowing for even more flexibility.
HL-MRFs are a significant addition to the the library of machine learning tools because they embody
a useful point in the spectrum of models that trade off between scalability and expressivity.
As we showed, they can be easily applied to a wide range of structured problems in machine learning
and achieve high-quality predictive performance, competitive with or surpassing the performance of
canonical approaches.
However, these other models either do not scale as well, like discrete MRFs, or are not as versatile in their
ability to capture a wide range of problems, like Bayesian probabilistic matrix factorization.

We also introduced PSL, a probabilistic programming language for HL-MRFs.
PSL makes HL-MRFs easy to design, allowing users to encode their ideas for structural dependencies
using an intuitive syntax based on first-order logic.
PSL also helps accelerate a time-consuming aspect of the modeling process: refining a model.
In contrast with other types of models that require specialized inference and learning algorithms
depending on which structural dependencies are included, HL-MRFs can encode many types of
dependencies and scale well with the same inference and learning algorithms.
PSL makes it easy to quickly add, remove, and modify dependencies in the model and rerun inference
and learning, allowing users to quickly improve the quality of their models.
Finally, because PSL uses a first-order syntax, each PSL program actually specifies an entire class of HL-MRFs, parameterized by the
particular data set over which it is grounded.
Therefore, a model or components of a model refined for one data set can easily be applied to others.

Next, we introduced inference and learning algorithms that scale to large problems.
The MAP inference algorithm is far more scalable than standard tools for convex optimization
because it leverages the sparsity that is so common to the dependencies in structured prediction.
The supervised learning algorithms extend standard learning objectives to HL-MRFs.
Together, this combination of an expressive formalism, a user-friendly probabilistic programming
language, and highly scalable algorithms enables researchers and practitioners to easily build large-scale,
accurate models of relational and structured data.\footnote{An open source implementation, tutorials, and data sets are available at \url{http://psl.linqs.org}.}

This paper also lays the foundation for many lines of future work.
Our analysis of local consistency relaxation (LCR) as a hierarchical optimization is a general proof
technique,
and it could be used to derive compact forms for other LCR objectives.
As in the case of MRFs defined using logical clauses, such compact forms can simplify analysis and could
lead to a greater understanding of LCR for other classes of MRFs.
Another important line of work is understanding what guarantees apply to the MAP states of HL-MRFs.
Can anything be said about their ability to approximate MAP inference in discrete models that go beyond the models
already covered by the known rounding guarantees?
Future directions also include developing new algorithms for HL-MRFs.
One important direction is marginal inference for HL-MRFs and algorithms for sampling from them.
Unlike marginal inference for discrete distributions, which computes the marginal probability that a variable is in a particular state,
marginal inference for HL-MRFs requires finding the marginal probability that a variable is in a particular range.
One option for doing so, as well as generating samples from HL-MRFs, is to extend the hit-and-run sampling scheme of 
\citet{broecheler:nips10}.
This method was developed for continuous constrained MRFs with piecewise-linear potentials.
There are also many new domains to which HL-MRFs and PSL can be applied.
With these modeling tools, researchers can design and apply new
solutions to structured prediction problems.


\acks{

We acknowledge the many people who have contributed to the development
of HL-MRFs and PSL.
Contributors include Eriq Augustine, Shobeir Fakhraei, James Foulds, Angelika Kimmig, Stanley Kok, Ben London, Hui Miao, Lilyana Mihalkova,
Dianne P. O'Leary, Jay Pujara, Arti Ramesh, Theodoros Rekatsinas, and V.S. Subrahmanian.
This work was supported by NSF grants CCF0937094 and IIS1218488, and IARPA via DoI/NBC contract number D12PC00337.
The U.S. Government is authorized to reproduce and distribute reprints for governmental purposes notwithstanding any
copyright annotation thereon.
Disclaimer: The views and conclusions contained herein are those of the authors and should not be interpreted as necessarily
representing the official policies or endorsements, either expressed or implied, of IARPA, DoI/NBC, or the U.S. Government.}

\appendix


\section{Proof of Theorem~\ref{thm:equivalence}}
\label{app:equivalence}

In this appendix, we prove the equivalence of objectives~(\ref{eq:lpmaxsat}) and~(\ref{eq:localvariational}).
Our proof analyzes the local consistency relaxation
to derive an equivalent, more compact optimization over only the variable pseudomarginals $\ppmlpvar$
that is identical to the MAX SAT relaxation.
Since the variables are Boolean, we refer to each pseudomarginal
$\pmlpvar_\imsvar(1)$ as simply $\pmlpvar_\imsvar$.
Let $\msvarset_\imsclause^\false$ denote the unique setting such that
$\pot_\imsclause(\msvarset_\imsclause^\false) = 0$.
(I.e., $\msvarset_\imsclause^\false$ is the setting in which each literal in the clause
$\msclause_\imsclause$ is false.)

We begin by reformulating the local consistency relaxation as a hierarchical optimization, first over
the variable pseudomarginals $\ppmlpvar$ and then over the factor pseudomarginals $\ppmlpfactorvar$.
Due to the structure of local polytope $\localpolytope$, the pseudomarginals $\ppmlpvar$ parameterize inner linear
programs that decompose over the structure of the MRF,
such that---given fixed $\ppmlpvar$---there is an independent linear program
$\pmlppot_\imsclause(\ppmlpvar)$ over $\ppmlpfactorvar_\imsclause$ for each clause $\msclause_\imsclause$.
We rewrite objective~(\ref{eq:localvariational}) as
\begin{equation}
\label{eq:logicvariational}
\argmax_{\ppmlpvar \in [0,1]^\ny} \sum_{\msclause_\imsclause \in \msclauseset} \pmlppot_\imsclause(\ppmlpvar),
\end{equation}
where 
\begin{align}
\pmlppot_\imsclause(\ppmlpvar) =
\max_{\ppmlpfactorvar_\imsclause} & ~~\msweight_\imsclause \sum_{\msvarset_\imsclause | \msvarset_\imsclause \neq \msvarset_\imsclause^\false} \pmlpfactorvar_\imsclause(\msvarset_\imsclause) & \label{eq:pmlp}\\
\text{such that}
& \sum_{\msvarset_\imsclause | \msvar_\imsclause(\imsvar) = 1} \pmlpfactorvar_\imsclause(\msvarset_\imsclause) = \pmlpvar_\imsvar & \forall \imsvar \in \msindicatorset_\imsclause^+ \label{eq:posmarginal} \\
& \sum_{\msvarset_\imsclause | \msvar_\imsclause(\imsvar) = 0} \pmlpfactorvar_\imsclause(\msvarset_\imsclause) = 1 - \pmlpvar_\imsvar & \forall \imsvar \in \msindicatorset_\imsclause^- \label{eq:negmarginal} \\
& ~~~\sum_{\msvarset_\imsclause} \pmlpfactorvar_\imsclause(\msvarset_\imsclause) = 1 & \label{eq:simplex} \\
& ~~~\pmlpfactorvar_\imsclause(\msvarset_\imsclause) \geq 0 & \forall \msvarset_\imsclause~. \label{eq:primalnonneg}
\end{align}
It is straightforward to verify that objectives~(\ref{eq:localvariational}) and~(\ref{eq:logicvariational})
are equivalent for MRFs with disjunctive clauses for potentials.
All constraints defining $\localpolytope$ can be derived from the constraint
$\ppmlpvar \in [0,1]^\nmsvar$ and the constraints in the definition of $\pmlppot_\imsclause(\ppmlpvar)$.
We have omitted redundant constraints to simplify analysis.

To make this optimization more compact, we replace each inner linear program
$\pmlppot_\imsclause(\ppmlpvar)$ with an expression that gives its optimal value for any setting of $\ppmlpvar$.
Deriving this expression requires reasoning about any maximizer $\ppmlpfactorvar_\imsclause^\star$ of
$\pmlppot_\imsclause(\ppmlpvar)$, which is guaranteed to exist because problem~(\ref{eq:pmlp}) is
bounded and feasible\footnote{Setting $\pmlpfactorvar_\imsclause(\msvarset_\imsclause)$
to the probability defined by $\ppmlpvar$ under the assumption that the elements of
$\msvarset_\imsclause$ are independent, i.e., the product of the pseudomarginals,
is always feasible.} for any parameters $\ppmlpvar \in [0,1]^\nmsvar$ and $\msweight_\imsclause$.

We first derive a sufficient condition for the linear program to not be fully satisfiable, in the sense that it
cannot achieve a value of $\msweight_\imsclause$, the maximum value of the weighted potential
$\msweight_\imsclause \pot_\imsclause(\msvarset)$.
Observe that, by the objective~(\ref{eq:pmlp}) and the simplex constraint~(\ref{eq:simplex}), showing
that $\pmlppot_\imsclause(\ppmlpvar)$ is not fully satisfiable is equivalent to showing that
$\pmlpfactorvar^\star_\imsclause(\msvarset^\false_\imsclause) > 0$.
\begin{lemma}
\label{le:unsat}
If
\[
\sum_{\iy \in \msindicatorset_\ipot^+} \pmlpvar_\iy + \sum_{\iy \in \msindicatorset_\ipot^-} (1 -\pmlpvar_\iy) < 1~,
\]
then $\pmlpfactorvar^\star_\imsclause(\msvarset^\false_\imsclause) > 0$.
\end{lemma}
\begin{proof}
By the simplex constraint~(\ref{eq:simplex}),
\[
\sum_{\iy \in \msindicatorset_\ipot^+} \pmlpvar_\iy + \sum_{\iy \in \msindicatorset_\ipot^-} (1 -\pmlpvar_\iy)
< \sum_{\msvarset_\imsclause} \pmlpfactorvar^\star_\imsclause(\msvarset_\imsclause)~.
\]
Also, by summing all the constraints~(\ref{eq:posmarginal}) and~(\ref{eq:negmarginal}),
\[
\sum_{\msvarset_\imsclause | \msvarset_\imsclause \neq \msvarset_\imsclause^\false}
\pmlpfactorvar^\star_\imsclause(\msvarset_\imsclause)
\leq \sum_{\iy \in \msindicatorset_\ipot^+} \pmlpvar_\iy + \sum_{\iy \in \msindicatorset_\ipot^-} (1 -\pmlpvar_\iy)~,
\]
because all the components of $\ppmlpfactorvar^\star$ are nonnegative, and---except for
$\pmlpfactorvar^\star_\imsclause(\msvarset^\false_\imsclause)$---they all appear at least once in
constraints~(\ref{eq:posmarginal}) and~(\ref{eq:negmarginal}).
These bounds imply
\[
\sum_{\msvarset_\imsclause | \msvarset_\imsclause \neq \msvarset_\imsclause^\false}
\pmlpfactorvar^\star_\imsclause(\msvarset_\imsclause)
< \sum_{\msvarset_\imsclause} \pmlpfactorvar^\star_\imsclause(\msvarset_\imsclause)~,
\]
which means $\pmlpfactorvar^\star_\imsclause(\msvarset^\false_\imsclause) > 0$, completing the proof.
\end{proof}

We next show that if $\pmlppot_\imsclause(\ppmlpvar)$ is parameterized such that it is not fully
satisfiable, as in Lemma \ref{le:unsat}, then its optimum always takes a particular value defined by $\ppmlpvar$.
\begin{lemma}
\label{le:objunsat}
If $\param_\ipot > 0$ and $\pmlpfactorvar^\star_\imsclause(\msvarset^\false_\imsclause) > 0$,
then
\[
\sum_{\msvarset_\imsclause | \msvarset_\imsclause \neq \msvarset_\imsclause^\false}
\pmlpfactorvar^\star_\imsclause(\msvarset_\imsclause)
= \sum_{\iy \in \msindicatorset_\ipot^+} \pmlpvar_\iy + \sum_{\iy \in \msindicatorset_\ipot^-} (1 -\pmlpvar_\iy)~.
\]
\end{lemma}
\begin{proof}
We prove the lemma via the Karush-Kuhn-Tucker (KKT) conditions \citep{karush:thesis39, kuhn:berkeley51}.
Since problem~(\ref{eq:pmlp}) is a maximization of a linear function subject to linear constraints,
the KKT conditions are necessary and sufficient for any optimum $\ppmlpfactorvar_\imsclause^\star$.

Before writing the relevant KKT conditions, we introduce some necessary notation.
For a state $\msvarset_\imsclause$, we need to reason about the variables that disagree with
the unsatisfied state $\msvarset^\false_\imsclause$.
Let
\[
\statedisagrees(\msvarset_\imsclause) \triangleq
\left\{\imsvar \in \msindicatorset_\imsclause^+ \cup \msindicatorset_\imsclause^-
| \msvarset_\imsclause(i) \neq \msvarset^\false_\imsclause(i) \right\}
\]
be the set of indices for the variables that do not have the same
value in the two states $\msvarset_\imsclause$ and $\msvarset^\false_\imsclause$.

We now write the relevant KKT conditions for $\ppmlpfactorvar_\imsclause^\star$.
Let ${\boldsymbol \eqkkt}, {\boldsymbol \ineqkkt}$ be real-valued vectors
where $|{\boldsymbol \eqkkt}| = |\msindicatorset_\imsclause^+| + |\msindicatorset_\imsclause^-| + 1$
and $|{\boldsymbol \ineqkkt}| = |\ppmlpfactorvar_\imsclause|$.
Let each $\eqkkt_\imsvar$ correspond to a constraint~(\ref{eq:posmarginal}) or~(\ref{eq:negmarginal})
for $\imsvar \in \msindicatorset_\imsclause^+ \cup \msindicatorset_\imsclause^-$,
and let $\eqkkt_\Delta$ correspond to the simplex constraint~(\ref{eq:simplex}).
Also, let each $\ineqkkt_{\msvarset_\imsclause}$ correspond to a constraint~(\ref{eq:primalnonneg})
for each $\msvarset_\imsclause$.
Then, the following KKT conditions hold:
\begin{align}
& \ineqkkt_{\msvarset_\imsclause} \geq 0 & \forall \msvarset_\imsclause \label{eq:kktnonneg} \\
& \ineqkkt_{\msvarset_\imsclause} \pmlpfactorvar^\star_\imsclause(\msvarset_\imsclause) = 0 & \forall \msvarset_\imsclause \label{eq:ineqcomp} \\
& \eqkkt_\Delta + \ineqkkt_{\msvarset^\false_\imsclause} = 0 & \label{eq:falsestat} \\
& \msweight_\imsclause + \sum_{\imsvar \in \statedisagrees(\msvarset_\imsclause)} \eqkkt_\imsvar + \eqkkt_\Delta + \ineqkkt_{\msvarset_\imsclause} = 0
& \forall \msvarset_\imsclause \neq \msvarset^\false_\imsclause~. \label{eq:stat}
\end{align}

Since $\pmlpfactorvar^\star_\imsclause(\msvarset^\false_\imsclause) > 0$, by condition~(\ref{eq:ineqcomp}),
$\ineqkkt_{\msvarset^\false_\imsclause} = 0$.
By condition~(\ref{eq:falsestat}), then $\eqkkt_\Delta = 0$.
From here we can bound the other elements of ${\boldsymbol \eqkkt}$.
Observe that for every
$\imsvar~\in~\msindicatorset_\imsclause^+~\cup~\msindicatorset_\imsclause^-$,
there exists a state $\msvarset_\imsclause$ such that
$\statedisagrees(\msvarset_\imsclause) = \{\imsvar\}$.
Then, it follows from condition~(\ref{eq:stat}) that there exists $\msvarset_\imsclause$ such that,
for every $\imsvar \in \msindicatorset_\imsclause^+ \cup \msindicatorset_\imsclause^-$,
\[
\msweight_\imsclause + \eqkkt_\imsvar + \eqkkt_\Delta + \ineqkkt_{\msvarset_\imsclause} = 0~.
\]
Since $\ineqkkt_{\msvarset_\imsclause} \geq 0$ by condition~(\ref{eq:kktnonneg})
and $\eqkkt_\Delta = 0$, it follows that $\eqkkt_\imsvar \leq -\msweight_\imsclause$.
With these bounds, we show that, for any state $\msvarset_\imsclause$,
if $|\statedisagrees(\msvarset_\imsclause)| \geq 2$,
then $\pmlpfactorvar^\star_\imsclause(\msvarset_\imsclause) = 0$.
Assume that for some state $\msvarset_\imsclause$, $|\statedisagrees(\msvarset_\imsclause)| \geq 2$.
By condition~(\ref{eq:stat}) and the derived constraints on ${\boldsymbol \eqkkt}$,
\[
\ineqkkt_{\msvarset_\imsclause} \geq (|\statedisagrees(\msvarset_\imsclause)| - 1) \msweight_\imsclause > 0~.
\]
With condition~(\ref{eq:ineqcomp}), $\pmlpfactorvar^\star_\imsclause(\msvarset_\imsclause) = 0$.
Next, observe that for all $\imsvar \in \msindicatorset_\imsclause^+$
(resp.\ $\imsvar \in \msindicatorset_\imsclause^-$) and for any state $\msvarset_\imsclause$,
if $\statedisagrees(\msvarset_\imsclause) = \{\imsvar\}$, then $\msvar_\imsclause(\imsvar) = 1$
(resp.\ $\msvar_\imsclause(\imsvar) = 0$), and for any other state $\msvarset_\imsclause^\prime$
such that $\msvar_\imsclause^\prime(\imsvar)=1$ (resp.\ $\msvar_\imsclause^\prime(\imsvar)=0$),
$\statedisagrees(\msvarset_\imsclause^\prime) \geq 2$.
By constraint~(\ref{eq:posmarginal}) (resp.\ constraint~(\ref{eq:negmarginal})),
$\pmlpfactorvar^\star(\msvarset_\imsclause) = \pmlpvar_\imsvar$
(resp.\ $\pmlpfactorvar^\star(\msvarset_\imsclause) = 1 - \pmlpvar_\imsvar$).

We have shown that if $\pmlpfactorvar^\star_\imsclause(\msvarset^\false_\imsclause) > 0$,
then for all states $\msvarset_\imsclause$, if $\statedisagrees(\msvarset_\imsclause) = \{\imsvar\}$
and $\imsvar \in \msindicatorset_\imsclause^+$
(resp.\ $\imsvar \in \msindicatorset_\imsclause^-$),
then $\pmlpfactorvar^\star_\imsclause(\msvarset_\imsclause) = \pmlpvar_\imsvar$
(resp.\ $\pmlpfactorvar^\star_\imsclause(\msvarset_\imsclause) = 1 - \pmlpvar_\imsvar$),
and if $|\statedisagrees(\msvarset_\imsclause)| \geq 2$,
then $\pmlpfactorvar^\star_\imsclause(\msvarset_\imsclause) = 0$.
This completes the proof.
\end{proof}

Lemma~\ref{le:unsat} says if $\sum_{\iy \in \msindicatorset_\ipot^+} \pmlpvar_\iy + \sum_{\iy \in \msindicatorset_\ipot^-} (1 -\pmlpvar_\iy) < 1$, then $\pmlppot_\imsclause(\ppmlpvar)$
is not fully satisfiable, and Lemma~\ref{le:objunsat} provides its optimal value.
We now reason about the other case, when $\sum_{\iy \in \msindicatorset_\ipot^+} \pmlpvar_\iy + \sum_{\iy \in \msindicatorset_\ipot^-} (1 -\pmlpvar_\iy) \geq 1$, and we show that this condition is sufficient to ensure that
 $\pmlppot_\imsclause(\ppmlpvar)$ is fully satisfiable.
\begin{lemma}
\label{le:sat}
If $\msweight_\ipot > 0$ and
\[
\sum_{\iy \in \msindicatorset_\ipot^+} \pmlpvar_\iy + \sum_{\iy \in \msindicatorset_\ipot^-} (1 -\pmlpvar_\iy) \geq 1~,
\]
then $\pmlpfactorvar^\star_\imsclause(\msvarset^\false_\imsclause) = 0$.
\end{lemma}
\begin{proof}
We prove the lemma by contradiction.
Assume that $\msweight_\ipot > 0$,
$\sum_{\iy \in \msindicatorset_\ipot^+} \pmlpvar_\iy + \sum_{\iy \in \msindicatorset_\ipot^-} (1 -\pmlpvar_\iy) \geq 1$,
and that the lemma is false, $\pmlpfactorvar^\star_\imsclause(\msvarset^\false_\imsclause) > 0$.
Then, by Lemma~\ref{le:objunsat},
\[
\sum_{\msvarset_\imsclause | \msvarset_\imsclause \neq \msvarset_\imsclause^\false}
\pmlpfactorvar^\star_\imsclause(\msvarset_\imsclause) \geq 1~.
\]
The assumption that $\pmlpfactorvar^\star_\imsclause(\msvarset^\false_\imsclause) > 0$ implies
\[
\sum_{\msvarset_\imsclause} \pmlpfactorvar^\star_\imsclause(\msvarset_\imsclause) > 1,
\]
which is a contradiction, since it violates the simplex constraint~(\ref{eq:simplex}).
The possibility that $\pmlpfactorvar^\star_\imsclause(\msvarset^\false_\imsclause) < 0$ is excluded by the nonnegativity 
constraints~(\ref{eq:primalnonneg}).
\end{proof}
For completeness and later convenience, we also state the value of $\pmlppot_\imsclause(\ppmlpvar)$
when it is fully satisfiable.
\begin{lemma}
\label{le:objsat}
If $\pmlpfactorvar^\star_\imsclause(\msvarset^\false_\imsclause) = 0$,
then
\[
\sum_{\msvarset_\imsclause | \msvarset_\imsclause \neq \msvarset_\imsclause^\false}
\pmlpfactorvar^\star_\imsclause(\msvarset_\imsclause) = 1~.
\]
\end{lemma}
\begin{proof}
The lemma follows from the simplex constraint~(\ref{eq:simplex}).
\end{proof}

We can now combine the previous lemmas into a single expression for the value of
$\pmlppot_\imsclause(\ppmlpvar)$.
\begin{lemma}
\label{le:obj}
For any feasible setting of $\ppmlpvar$,
\[
\pmlppot_\imsclause(\ppmlpvar)
= \msweight_\imsclause \min \left\{ \sum_{\imsvar \in \msindicatorset_\imsclause^+}\pmlpvar_\imsvar + \sum_{\imsvar \in \msindicatorset_\imsclause^-}(1-\pmlpvar_\imsvar), 1 \right\}~.
\]
\end{lemma}
\begin{proof}
The lemma is trivially true if $\msweight_\imsclause=0$ since any assignment will yield zero value.
If $\msweight_\imsclause~>~0$, then we consider two cases.
In the first case, if $\sum_{\iy \in \msindicatorset_\ipot^+} \pmlpvar_\iy + \sum_{\iy \in \msindicatorset_\ipot^-} (1 -\pmlpvar_\iy) < 1$,
then, by Lemmas~\ref{le:unsat}~and~\ref{le:objunsat},
\[
\pmlppot_\imsclause(\ppmlpvar)
= \msweight_\imsclause \left( \sum_{\imsvar \in \msindicatorset_\imsclause^+}\pmlpvar_\imsvar
+ \sum_{\imsvar \in \msindicatorset_\imsclause^-}(1-\pmlpvar_\imsvar) \right)~.
\]
In the second case, if
$\sum_{\iy \in \msindicatorset_\ipot^+} \pmlpvar_\iy + \sum_{\iy \in \msindicatorset_\ipot^-} (1 -\pmlpvar_\iy) \geq 1$,
then, by Lemmas~\ref{le:sat}~and~\ref{le:objsat},
\[
\pmlppot_\imsclause(\ppmlpvar) = \msweight_\imsclause~.
\]
By factoring out $\msweight_\imsclause$, we can rewrite this piecewise definition of
$\pmlppot_\imsclause(\ppmlpvar)$ as $\msweight_\imsclause$ multiplied by the minimum of
$\sum_{\iy \in \msindicatorset_\ipot^+} \pmlpvar_\iy
+ \sum_{\iy \in \msindicatorset_\ipot^-} (1 -\pmlpvar_\iy)$
and $1$, completing the proof.
\end{proof}

This leads to our final equivalence result.
\begin{customthm}{\ref{thm:equivalence}}
For an MRF with potentials corresponding to disjunctive logical clauses and associated nonnegative
weights, the first-order local consistency relaxation of MAP inference is equivalent to the
MAX SAT relaxation of \citet{goemans:discmath94}.
Specifically, any partial optimum $\ppmlpvar^\star$ of objective~(\ref{eq:localvariational}) is an optimum
$\mssoftvarset^\star$ of objective~(\ref{eq:lpmaxsat}), and vice versa.
\end{customthm}
\begin{proof}
Substituting the solution of the inner optimization from Lemma~\ref{le:obj} into the local consistency relaxation
objective~(\ref{eq:logicvariational}) gives a projected optimization over only $\ppmlpvar$
which is identical to the MAX SAT relaxation objective~(\ref{eq:lpmaxsat}).
\end{proof}

\vskip 0.2in
\bibliography{15-631}

\end{document}